\documentclass{article}

\pdfoutput=1

    \PassOptionsToPackage{numbers, compress}{natbib}

     \usepackage[final]{neurips_2020}

\usepackage[utf8]{inputenc} 
\usepackage[T1]{fontenc}    
\usepackage{hyperref}       
\usepackage{url}            
\usepackage{booktabs}       
\usepackage{nicefrac}       
\usepackage{microtype}      

\usepackage{nccmath}
\usepackage{amsmath}
\usepackage{amssymb}
\usepackage{amsthm}
\usepackage{bm}
\usepackage{tikz}
\usetikzlibrary{bayesnet}
\usepackage{subcaption}
\usepackage{tabularx}
\usepackage{mathtools}
\usepackage{cleveref}
\usepackage{thmtools, thm-restate}
\usepackage{multirow}
\usepackage{changepage} 
\usepackage{adjustbox} 

\usepackage{lipsum}

\newtheorem{theorem}{Theorem}
\newtheorem{proposition}[theorem]{Proposition}
\newtheorem{definition}[theorem]{Definition}
\newtheorem{example}[theorem]{Example}
\newtheorem{remark}[theorem]{Remark}
\crefformat{section}{\S#2#1#3} 
\crefformat{subsection}{\S#2#1#3}
\crefformat{subsubsection}{\S#2#1#3}

\DeclarePairedDelimiterX\braket[2]{\langle}{\rangle}{#1 \delimsize\vert #2}

\DeclareMathOperator*{\argmin}{arg\,min}
\newcommand{\x}{\mathbf{x}}
\newcommand{\xF}{\mathbf{x}^\texttt{F}}
\newcommand{\xCE}{\mathbf{x}^\texttt{CE}}
\newcommand{\z}{\mathbf{z}}
\newcommand{\X}{\mathbf{X}}
\newcommand{\K}{\mathbf{K}}
\renewcommand{\k}{\mathbf{k}}
\newcommand{\Id}{\mathbf{I}}

\newcommand{\E}{\mathbb{E}}

\newcommand{\AF}{\mathbb{A}^\texttt{F}}
\newcommand{\aF}{a^\texttt{F}}
\newcommand{\R}{\mathbb{R}}

\newcommand{\M}{\mathcal{M}}
\newcommand{\N}{\mathcal{N}}

\renewcommand{\ss}{\sigma^2}

\newcommand{\nd}{\text{nd}}
\renewcommand{\d}{\text{d}}
\newcommand{\pa}{\text{pa}}
\newcommand{\dist}{\text{dist}}
\newcommand{\costF}{\text{cost}^\texttt{F}}
\newcommand{\subto}{\text{subject to}}
\newcommand{\U}{\mathbf{U}}
\renewcommand{\u}{\mathbf{u}}

\newcommand{\G}{\mathcal{G}}
\newcommand{\I}{\mathcal{I}}
\renewcommand{\th}{{\bm\theta}}

\newcommand{\muF}{\mu^\texttt{F}}
\newcommand{\ssF}{s^\texttt{F}}
\newcommand{\xSCF}{\mathbf{x}^\texttt{SCF}}
\newcommand{\XSCF}{\mathbf{X}^\texttt{SCF}}
\newcommand{\SH}{\mathbf{S}}
\newcommand{\MH}{\mathcal{M}}
\newcommand{\Mstar}{\mathcal{M}_\star}
\newcommand{\Sstar}{\mathbf{S}_\star}
\newcommand{\CATEstar}{\textsc{cate}_\star}
\newcommand{\Validstar}{\text{Valid}_\star}
\newcommand{\lin}{\textsc{lin}}
\newcommand{\KR}{\textsc{kr}}
\newcommand{\KLD}[2]{D_{\mathrm{KL}} \left( \left. \left. #1 \right|\right| #2 \right) }

\newcommand{\f}{\mathbf{f}}
\newcommand{\given}{\vert}
\newcommand{\0}{\mathbf{0}}
\newcommand{\LCB}{\textsc{lcb}}
\newcommand{\fig}{Fig.~}
\newcommand{\norm}[1]{\left\lVert#1\right\rVert}
\newcommand{\addRef}[1][]{{\color{red} (Add ref)}}
\newcommand{\cate}{\textsc{cate}}
\newcommand{\cvae}{\textsc{cvae}}

\newcommand{\gp}{\textsc{gp}}
\newcommand{\gpscm}{\textsc{gp-scm}}
\newcommand{\scm}{\textsc{scm}}
\newcommand{\anm}{\textsc{anm}}
\newcommand{\enc}{E}
\newcommand{\dec}{D}

\title{Algorithmic recourse under imperfect causal knowledge: a probabilistic approach}

\author{%
  Amir-Hossein Karimi\footnotemark[1]\, $^{1,2}$\quad
  Julius von K\"ugelgen\thanks{Equal contribution}\, $^{1,3}$\quad
  Bernhard Sch\"olkopf $^{1}$\quad
  Isabel Valera $^{1,4}$ \\ \\
  $^1$Max Planck Institute for Intelligent Systems, T\"ubingen, Germany \\
  $^2$Max Planck ETH Center for Learning Systems, Z\"urich, Switzerland \\
  $^3$Department of Engineering, University of Cambridge, United Kingdom \\
  $^4$Department of Computer Science, Saarland University, Saarbr\"ucken, Germany \\
  \texttt{\{amir, jvk, bs, ivalera\}@tue.mpg.de}
}

\begin{document}

\renewcommand{\baselinestretch}{1}
\newcommand{\negspace}{0em}

\maketitle
\begin{abstract}
Recent work has discussed the limitations of counterfactual explanations to recommend actions for algorithmic recourse, and argued for the need of taking causal relationships between features into consideration.
Unfortunately, in practice, the true underlying
structural
causal model 
is generally unknown.
In this work, we first show that it is impossible to guarantee recourse without access to the true structural equations.
To address this limitation, we propose two probabilistic approaches to select optimal actions that achieve recourse with  high probability
given limited causal knowledge (e.g., only the causal graph). 
The first captures uncertainty over structural equations under additive Gaussian noise, and uses Bayesian model averaging to estimate the counterfactual distribution. 
The second removes any assumptions on the structural equations by instead computing the average effect of recourse actions on individuals similar to the person who seeks recourse, leading to a novel subpopulation-based interventional notion of recourse.
We then derive a gradient-based procedure for selecting optimal recourse actions, and empirically show that the proposed approaches lead to more reliable  recommendations under imperfect 
causal knowledge than non-probabilistic baselines.
\end{abstract}

\vspace{\negspace}
\section{Introduction}
\label{sec:introduction}
\vspace{\negspace}
As machine learning algorithms are increasingly used to assist consequential decision making in a wide range of real-world settings~\citep{perry2013predictive, romero2011preface}, providing explanations for the decision of these black-box models becomes crucial~\citep{bhatt2020explainable, wachter2017right}.
A popular approach is that of (nearest) \textit{counterfactual explanations}, which refer to the closest feature instantiations that would have resulted in a changed prediction~\citep{wachter2017counterfactual}.
While providing some insight (explanation) into the underlying black-box classifier, such counterfactual explanations do not directly translate into actionable recommendations to individuals for obtaining a more favourable prediction\cite{karimi2020algorithmic, barocas2020hidden}---a related task referred to as \textit{algorithmic recourse}~\citep{ustun2019actionable,venkatasubramanian2020philosophical, joshi2019towards,karimi2020survey}.
Importantly, prior work on both counterfactual explanations and algorithmic recourse treats features as independently manipulable inputs,
thus \textit{ignoring the causal relationships between features}.

In this context, recent work~\citep{karimi2020algorithmic} has argued for the need of taking into account the causal structure between features to find a minimal set of actions (in the form of interventions) that guarantees recourse. 
However, while this approach 
is theoretically sound, it involves computing counterfactuals in the \emph{true underlying structural causal model} (\scm) \cite{pearl2009causality}, and thus relies on {strong impractical assumptions};
specifically, 
it requires complete knowledge
of the true structural equations.
While for many applications it is possible to draw a causal diagram from expert knowledge, assumptions about the form of structural equations are, in general, not testable and may thus not hold in practice~\citep{peters2017elements}.
 As a result, counterfactuals computed using a misspecified causal model may be inaccurate and recommend actions that are sub-optimal or, even worse, ineffective to achieve recourse.


In this work, we focus on the problem of algorithmic recourse when only \textit{limited causal knowledge} is available (as it is generally the case). 
To this end, we propose two probabilistic approaches which allow to relax the strong assumption of a fully-specified \scm\ made in \cite{karimi2020algorithmic}. 
In the first approach, we assume that, while the underlying \scm{} is unknown, it belongs to the family of additive Gaussian noise models~\citep{hoyer2009nonlinear, peters2014identifiability}.
We then make use of Gaussian processes (\gp s) \citep[][]{williams2006gaussian} to average predictions over a whole 
family of \scm{}s and thus to obtain a distribution over \textit{counterfactual} outcomes which forms the basis for \emph{individualised} algorithmic recourse. %
%
%
The second approach considers a different \emph{subpopulation-based} notion of algorithmic recourse by estimating 
the effect of \textit{interventions} for  individuals similar to the one for which we aim to achieve recourse.
It thus addresses a different (rung 2) target quantity than the counterfactual/individualised (rung 3) approach which allows us to further relax our assumptions by removing any assumptions on the form of the structural equations. %
This approach is based on the idea of the conditional average treatment effect (\cate) \citep[][]{abrevaya2015estimating}, and relies on conditional variational autoencoders (\cvae s)~\citep[][]{sohn2015learning} to estimate the interventional distribution.
In both cases, we assume that the causal graph is known or can be postulated from expert knowledge, as without such an assumption
causal reasoning from observational data is 
 not possible \citep[][Prop.\ 4.1]{peters2017elements}.


In more detail, 
we first demonstrate as a motivating negative result that recourse guarantees are only possible if the true \scm\ is known (\cref{sec:motivating_example}). 
Then, we introduce two probabilistic approaches for handling different levels of uncertainty in the structural equations (\cref{sec:recourse_via_probabilistic_counterfactuals} and \cref{sec:subpopulation_based_recourse_via_CATE}), and propose a gradient-based method
to find a set of actions that achieves recourse with a given probability at minimum cost (\cref{sec:optimisation_problem}). 
%
%
%
Our experiments (\cref{sec:experiments}) on synthetic and semi-synthetic loan approval data, show the need for probabilistic approaches to achieve algorithmic recourse in practice, as point estimates of the underlying true \scm{}  often propose invalid recommendations or achieve recourse only at higher cost. Importantly, our results also show that subpopulation-based recourse is the right approach to adopt when assumptions such as additive noise do not hold.
A user-friendly implementation of all methods that only requires specification of the causal graph and a training set is available at \url{https://github.com/amirhk/recourse}.


\vspace{\negspace}
\section{Background and related work}
\label{sec:background}
\vspace{\negspace}

\paragraph{Causality: structural causal models, interventions, and counterfactuals.}
\label{sec:background_causality}

To reason formally about causal relations between features $\X=\{X_1, ..., X_d\}$, we adopt the \emph{structural causal model} (\scm) framework~\cite{pearl2009causality}.\footnote{Also known as non-parametric structural equation model with independent errors (\textsc{npsem-ie}).}
Specifically, we assume that the data-generating process of $\X$ is described by an (unknown) underlying \scm\ $\M$ of the general form
\begin{equation}
    \label{eq:M0_oracle}
    \M=(\mathbf{S}, P_{\U}), \quad \mathbf{S}=\{X_r := f_r(\X_{\pa(r)}, U_r)\}_{r=1}^d, \quad P_\U = P_{U_1}\times \ldots\times P_{U_d},
\end{equation}
where the structural equations $\mathbf{S}$ are a set of assignments generating each observed 
variable $X_r$ as a deterministic function $f_r$ of its causal parents $\X_{\pa(r)}\subseteq \X \setminus X_r$ and an unobserved 
noise variable $U_r$. 
The assumption of mutually independent noises (i.e., a fully factorised $P_\U$)  entails that there is no hidden confounding and is referred to as \emph{causal sufficiency}.
An \scm\ is often illustrated by its associated causal graph $\G$,
which is obtained by drawing a directed edge from each node in $\X_{\pa(r)}$ to $X_r$ for $r\in[d]:=\{1,\ldots, d\}$, see \fig \ref{fig:SCM} and \ref{fig:causal_graph} for an example.
We assume  throughout that $\G$ is acyclic.
In this case, $\M$ implies a unique observational distribution $P_\X$, which factorises over $\G$, defined as the push-forward of $P_\U$ via $\mathbf{S}$.\footnote{I.e., for $r\in[d]$, $P_{X_r|\X_{\pa(r)}}(X_r|\X_{\pa(r)}):=P_{U_r}(f_r^{-1}(X_r|\X_{\pa(r)}))$, where $f_r^{-1}(X_r|\X_{\pa(r)})$ denotes the pre-image of $X_r$ given $\X_{\pa(r)}$ under $f_r$, i.e., $f_r^{-1}(X_r|\X_{\pa(r)}):=\{u\in\mathcal{U}_r:f_r(\X_{\pa(r)},u)=X_r\}$.}

Importantly, the \scm\ framework also entails \textit{interventional distributions} describing a situation in which some variables are manipulated externally.
E.g., using the \textit{do}-operator, an intervention which fixes $\X_\I$ to $\th$ (where $\I\subseteq [d]$) is denoted by $do(\X_\I=\th)$.
The corresponding distribution of the remaining variables $\X_{-\I}$ can be computed 
by replacing the structural equations for $\X_\I$ in $\mathbf{S}$ to obtain the new set of equations $\mathbf{S}^{do(\X_\I=\th)}$. 
The interventional distribution $P_{\X_{-\I}|do(\X_\I=\th)}$ is then given by the observational distribution implied by the manipulated \scm\ $\left(\mathbf{S}^{do(\X_\I=\th)}, P_\U\right)$.

Similarly, an \scm\ also implies distributions over \textit{counterfactuals}---statements about a world in which a hypothetical intervention was performed \emph{all else being equal}.
For example, \emph{given} observation $\xF$ we can ask what would have happened if  $\X_\I$ had instead taken the value $\th$.
We denote the counterfactual variable by $\X(do(\X_\I=\th))|\xF$, whose distribution can be computed in three steps~\cite{pearl2009causality}:

\textit{1. Abduction}: compute the posterior distribution over background variables given  $\xF$, $P_{\U|\xF}$;\\
\textit{2. Action:} perform the intervention to obtain the new structural equations $\mathbf{S}^{do(\X_\I=\th)}$; and,\\
\textit{3. Prediction:} 
$P_{\X(do(\X_\I=\th))|\xF}$
is the distribution induced by 
the resulting \scm\   
$\left(\mathbf{S}^{do(\X_\I=\th)}, P_{\U|\xF}\right)$.

\newcommand{\xshift}{2.25em}
\newcommand{\yshift}{2.75em}
\begin{figure}[tb]
    \begin{subfigure}[b]{0.3\textwidth}
        \centering
        \begin{tikzpicture}
            \centering
            \node (X_1) [latent] {$X_1$};
            \node (X_2) [latent, below=of X_1, xshift=-\xshift, yshift=\yshift] {$X_2$};
            \node (X_3) [latent, below=of X_1, xshift=\xshift, yshift=\yshift] {$X_3$};
            \node (h) [det, below=of X_2, xshift=\xshift, yshift=\yshift] {$h$};
            \edge{X_1, X_2, X_3}{h};
        \end{tikzpicture} 
        \caption{Classifier-centric view}
        \label{fig:independent_features}
    \end{subfigure}
    \begin{subfigure}[b]{0.4\textwidth}
        \centering
        \begin{align*}
            \mathbf{S}&=
            \left.
            \begin{cases}
                X_1 := f_1(U_1),\\
                X_2 := f_2(X_1, U_2),\\
                X_3 := f_3(X_1, X_2, U_3)
            \end{cases}
            \hspace{-1em}\right\}\\[.5em]
            P_\U&=P_{U_1}\times P_{U_2} \times P_{U_3}
        \end{align*}
        \caption{ $\M=(\mathbf{S}, P_\U)$}
        \label{fig:SCM}
    \end{subfigure}%
    \begin{subfigure}[b]{0.3\textwidth}
        \centering
        \begin{tikzpicture}
            \centering
            \node (X_1) [latent] {$X_1$};
            \node (X_2) [latent, below=of X_1, xshift=-\xshift, yshift=\yshift] {$X_2$};
            \node (X_3) [latent, below=of X_1, xshift=\xshift, yshift=\yshift] {$X_3$};
            \node (h) [det, below=of X_2, xshift=\xshift, yshift=\yshift] {$h$};
            \edge {X_1, X_2, X_3} {h};
            \edge{X_1}{X_2};
            \edge{X_1, X_2}{X_3};
        \end{tikzpicture}    
        \caption{Causal graph $\G$ for $\M$}
        \label{fig:causal_graph}
    \end{subfigure}
    \caption{A view commonly adopted for counterfactual explanations (a) treats features as independently manipulable inputs to a given fixed and deterministic classifier $h$. In the causal approach to algorithmic recourse taken in this work, we instead view variables as causally related to each other by a structural causal model (\scm) $\M$ (b) with associated causal graph $\G$ (c).}
    \label{fig:example}
\end{figure}
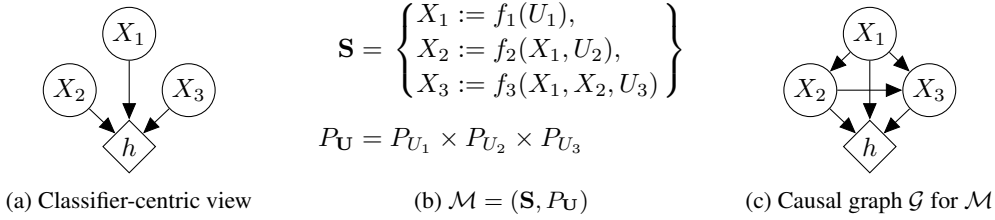

\paragraph{Explainable ML: ``counterfactual'' explanations and (causal) algorithmic recourse.}
\label{sec:background_XAI}
Assume that we are given a binary probabilistic classifier $h:\mathcal{X}\rightarrow[0,1]$ trained to make decisions about i.i.d.\ samples from the data distribution $P_{\X}$.\footnote{Following the related literature, we consider a binary classification task by convention; most of our considerations extend to multi-class classification or regression settings as well though.}
For ease of illustration, we adopt the setting of loan approval as a running example, i.e., $h(\x)\geq0.5$ denotes that a loan is granted and $h(\x)<0.5$ that it is denied. 
For a given individual $\xF$ that was denied a loan, $h(\xF)<0.5$, 
we aim to answer the following questions: ``Why did individual $\xF$ not get the loan?'' and ``What would they have to change, preferably with minimal effort, to increase their chances for a future application?''.

A popular approach to this task is to find so-called (nearest) \textit{counterfactual explanations}~\citep{wachter2017counterfactual}, where the term ``counterfactual'' is meant in the sense of the closest possible world with a different outcome~\citep{lewis1973counterfactuals}.
Translating this idea to our setting, a counterfactual explanation $\xCE$ for an individual  $\xF$ is given by a solution to the following optimisation problem:
\begin{equation}
    \label{eq:optimisation_problem_counterfactual_explanation}
    \textstyle
    \xCE \in \argmin_{\x \in \mathcal{X}} \quad \dist(\x, \xF) \quad \subto \quad h(\x)\geq0.5,
\end{equation}
where $\dist(\cdot,\cdot)$ is a similarity metric on $\mathcal{X}$, and  
additional constraints may be added to reflect
plausibility, feasibility, or diversity of the obtained counterfactual explanations~\citep{joshi2019towards, karimi2020model, mahajan2019preserving, mothilal2020explaining, poyiadzi2019face, sharma2020certifai}.
%

Importantly, while $\xCE$ signifies the most similar individual to $\xF$ that would receive the loan, it does not inform $\xF$ on the actions they should perform to become $\xCE$.
To address this limitation, 
the recently proposed framework of \emph{algorithmic recourse}  focuses instead on the actions an individual can perform to achieve a more favourable outcome \citep{ustun2019actionable}.
The emphasis is thus shifted from minimising a distance as in \eqref{eq:optimisation_problem_counterfactual_explanation} to optimising a personalised cost function $\costF(\cdot)$ over a set of actions $\AF$ which individual $\xF$ can perform. 
However, most prior work on both counterfactual explanations and algorithmic recourse
considers features as independently manipulable inputs to the classifier $h$ (see \fig\ref{fig:independent_features}), and therefore, ignores the potentially rich causal structure over $\X$ 
(see \fig \ref{fig:causal_graph}).
%
A number of authors have argued for the need to consider causal relations between variables when generative counterfactual explanations \cite{wachter2017counterfactual, ustun2019actionable, karimi2020model, mothilal2020explaining, mahajan2019preserving}, however, the resulting counterfactuals fail to imply feasible and optimal recourse actions~\cite{karimi2020algorithmic}.

In the most relevant work to the current~\cite{karimi2020algorithmic}, the authors approach the algorithmic recourse problem from a causal perspective within the \scm\ framework and propose to view recourse actions $a\in\AF$ as interventions of the form $do(\X_\I=\th)$.
For the class of invertible \scm s, such as  additive noise models (\anm)~\cite[]{hoyer2009nonlinear}, where the structural equations $\mathbf{S}$ are of the form 
\begin{equation}
    \textstyle
    \label{eq:ANM}
    \mathbf{S}=\{X_r:=f_r(\X_{\pa(r)})+U_r\}_{r=1}^d \quad \implies \quad  u_r^\texttt{F}=x_r^\texttt{F}-f_r(\x_{\pa(r)}^\texttt{F}), \quad r\in[d],
\end{equation}
they propose to use the three steps of structural counterfactuals in~\cite{pearl2009causality} to  assign a single counterfactual $\xSCF(a):=\x(a)|\xF$ to each action $a=do(\X_\I=\th)\in\AF$, and solve 
the optimisation problem,
\begin{equation}
\label{eq:optimisation_problem_recourse}
\textstyle
\aF= \argmin_{a=do(\X_\I=\th)\in\AF}\quad \costF(a) \quad \subto \quad h(\xSCF(a))\geq0.5.
\end{equation}

\vspace{\negspace}
\section{Negative result: no recourse guarantees for unknown structural equations}
\label{sec:motivating_example}
\vspace{\negspace}
In practice, the structural counterfactual $\xSCF(a)$ can only be computed using an approximate (and likely imperfect) \scm\ $\MH=(\SH, P_\U)$, which is estimated from data assuming a particular form of the structural equation as in \eqref{eq:ANM}. 
%
%
However, assumptions on the form of $\Sstar$ are generally untestable---not even with a randomised experiment---since there exist multiple \scm s which imply the same observational and interventional distributions, but entail different structural counterfactuals.

\begin{example}[adapted from 6.19 in \citep{peters2017elements}]
\label{example:equivalent_SCMs}
Consider the following two \scm s $\M_A$ and $\M_B$ which arise from the general form in Figure~\ref{fig:SCM} by choosing $U_1, U_2 \sim \text{Bernoulli}(0.5)$ and \hbox{$U_3\sim \text{Uniform}(\{0, \ldots, K\})$} independently in both $\M_A$ and $\M_B$, with  structural equations
\begin{align*}
        X_1 &:= U_1, &\text{in} \quad &\{\M_A, \M_B\},\\
        X_2 &:= X_1(1-U_2), &\text{in} \quad &\{\M_A, \M_B\},\\
        X_3 &:= \mathbb{I}_{X_1\neq X_2}(\mathbb{I}_{U_3>0}X_1+\mathbb{I}_{U_3=0}X_2) + \mathbb{I}_{X_1=X_2}U_3, &\text{in} \quad &\M_A,\\
        X_3 &:= \mathbb{I}_{X_1\neq X_2}(\mathbb{I}_{U_3>0}X_1+\mathbb{I}_{U_3=0}X_2) + \mathbb{I}_{X_1=X_2}(K-U_3), &\text{in} \quad &\M_B.
\end{align*}
Then $\M_A$ and $\M_B$ both imply exactly the same observational and interventional distributions, and  thus are indistinguishable from empirical data.
However, 
having observed $\xF=(1, 0, 0)$, they predict different counterfactuals had $X_1$ been $0$, i.e.,  $\xSCF(X_1=0)=(0,0,0)$ and $(0,0,K)$, respectively.\footnote{This follows from abduction on $\xF=(1, 0, 0)$ which for both $\M_A$ and $\M_B$ implies $U_3=0$.}
\end{example}

Confirming or refuting an assumed form of $\Sstar$ would thus require counterfactual data which is, by definition, never available.
Thus, example~\ref{example:equivalent_SCMs} proves the following proposition by contradiction.

\begin{proposition}[Lack of recourse guarantees]
\label{prop:no_guarantees}
Unless the set of descendants of intervened-upon variables is empty,
algorithmic recourse can, in general,  be guaranteed only if the true structural equations are known, irrespective of the amount 
and type
of available data.
\end{proposition}

\begin{remark}
The converse of Proposition \ref{prop:no_guarantees} does not hold. E.g., given $\xF=(1,0,1)$ in Example \ref{example:equivalent_SCMs}, abduction in either model yields $U_3>0$, so the counterfactual of $X_3$ cannot be predicted exactly.
\end{remark}


Building on the framework in \cite{karimi2020algorithmic}, we next present two novel approaches for causal algorithmic recourse 
under unknown structural equations.  
The first approach in \cref{sec:recourse_via_probabilistic_counterfactuals} aims to estimate the counterfactual distribution under the assumption of \anm s \eqref{eq:ANM} with Gaussian noise for the structural equations. 
The second approach in \cref{sec:subpopulation_based_recourse_via_CATE} makes no assumptions about the structural equations, and  instead of approximating the structural equations, it considers {the effect of interventions on a sub-population similar to} $\xF$. 
We recall that the causal graph is assumed to be known throughout.

\vspace{\negspace}
\section{Individualised algorithmic recourse via (probabilistic) counterfactuals}
\label{sec:recourse_via_probabilistic_counterfactuals}
\vspace{\negspace}
Since the true \scm\ $\Mstar$ is unknown, one approach to solving \eqref{eq:optimisation_problem_recourse} is to learn an approximate \scm\ $\MH$ within a given model class from training data $\{\x^i\}_{i=1}^n$.
For example, for an \anm\ \eqref{eq:ANM} with zero-mean noise, the functions $f_r$ can be learned via linear or kernel (ridge) regression of $X_r$ given $\X_{\pa(r)}$ as input. 
We refer to these approaches as $\MH_{\lin}$ and $\MH_{\KR}$, respectively. 
$\MH$ can then be used in place of $\Mstar$ to infer the noise values as in \eqref{eq:ANM}, and subsequently to predict a \textit{single-point counterfactual} $\xSCF(a)$ to be used in  \eqref{eq:optimisation_problem_recourse}.
%
However, the learned causal model $\MH$ may be imperfect, and thus lead to wrong counterfactuals  due to, e.g., the finite sample of the observed data, or more importantly, due to model misspecification (i.e., assuming a wrong parametric form for the structural equations). 
%
%
%

To solve such limitation, we adopt a Bayesian approach 
to account for the uncertainty in the estimation of the structural equations. 
Specifically, we assume  additive Gaussian noise and rely on  probabilistic regression using a Gaussian process (\gp) prior over the functions $f_r$ \citep{williams2006gaussian}.

\begin{definition}[\gpscm]
\label{def:GPSCM}
A Gaussian process \scm\ (\gpscm) over $\X$ refers to the model
\begin{equation}
    \label{eq:GP_SCM}
    X_r := f_r(\X_{\pa(r)})+ U_r, \quad\quad f_r\sim \mathcal{GP}(0, k_{r}), \quad\quad U_r\sim\mathcal{N}(0, \sigma^2_r), \quad\quad  r\in[d],
\end{equation}
with covariance functions $k_{r}:\mathcal{X}_{\pa(r)}\times\mathcal{X}_{\pa(r)}\rightarrow \R$, e.g., RBF kernels for continuous $X_{\pa(r)}$.
\end{definition}

While \gp s have previously been studied in a causal context for structure learning~\cite{friedman2000gaussian, von2019optimal}, estimating treatment effects~\citep{alaa2017bayesian, schulam2017reliable}, or learning \scm s with latent variables and measurement error~\citep{silva2010gaussian},
our goal here is to account for the uncertainty over $f_r$ in the computation of the posterior over  $U_r$, and thus to obtain a \emph{counterfactual distribution}, as summarised in the following  propositions.

\vspace{.5em}
\begin{restatable}[\gpscm\ noise posterior]{proposition}{GPSCMNP}
\label{prop:noise_posterior}
Let $\{\x^i\}_{i=1}^n$ be an observational sample from
\eqref{eq:GP_SCM}.
For each $r\in[d]$ with non empty parent set $|\pa(r)|>0$, the posterior distribution of the noise vector $\u_r=(u_r^1, ...,u_r^n)$,  conditioned on $\x_r=(x_r^1, ..., x_r^n)$ and $\X_{\pa(r)}=(\x_{\pa(r)}^1,...,\x_{\pa(r)}^n)$, is given by 
\begin{equation}
\label{eq:noise_posterior}
 \u_r|\X_{\pa(r)}, \x_r \sim \N\left(\ss_r (\K+\ss_r \Id)^{-1}\x_r, \ss_r\left(\Id-\ss_r(\K+\ss_r \Id)^{-1}\right)\right), 
\end{equation}
where $\K:=\big(k_r\big(\mathbf{x}_{\pa(r)}^i, \mathbf{x}_{\pa(r)}^j\big)\big)_{ij}$ denotes the Gram matrix.
\end{restatable}


Next, in order to compute counterfactual distributions, 
we rely on ancestral sampling (according to the causal graph) of the descendants of the intervention targets $\X_\I$ using the noise posterior of~\eqref{eq:noise_posterior}. 
The  counterfactual distribution of each descendant $X_r$ is  given by the following proposition.

\vspace{.5em}
\begin{restatable}[\gpscm\ counterfactual distribution]{proposition}{GPSCMCF}
\label{prop:counterfactual_distribution}
Let $\{\x^i\}_{i=1}^n$ be an observational sample from~\eqref{eq:GP_SCM}.
Then, for $r\in[d]$ with $|\pa(r)|>0$, the counterfactual distribution over $X_r$ had $\X_{\pa(r)}$ been $\tilde{\x}_{\pa(r)}$ (instead of $\xF_{\pa(r)}$) for individual $\xF\in \{\x^i\}_{i=1}^n$ is given by 
\begin{equation}
\label{eq:GP_SCM_counterfactual_dis}
X_r(\X_{\pa(r)}=\tilde{\x}_{\pa(r)})| \xF, \{\x^i\}_{i=1}^n \sim \N\big(\muF_r+\tilde{\k}^T(\K+\ss_r\Id)^{-1}\x_r,\,
\ssF_r + \tilde{k} - \tilde{\k}^T (\K+\ss_r\Id)^{-1} \tilde{\k} \big),
\end{equation}
where
$\tilde{k}:=k_r(\tilde{\x}_{\pa(r)}, \tilde{\x}_{\pa(r)})$, $\tilde{\k}:=\big(k_r(\tilde{\x}_{\pa(r)}, \x_{\pa(r)}^1), \ldots, k_r(\tilde{\x}_{\pa(r)}, \x_{\pa(r)}^n)\big)$, $\x_r$ and $\K$ as defined in Proposition \ref{prop:noise_posterior}, and $\muF_r$ and $\ssF_r$ are the posterior mean and variance of $u^\texttt{F}_r$ given by  \eqref{eq:noise_posterior}.
\end{restatable}

All proofs can be found in Appendix \ref{app:proofs}.
We can now generalise the recourse problem \eqref{eq:optimisation_problem_recourse} to our probabilistic setting by replacing the single-point counterfactual $\xSCF(a)$ with the counterfactual random variable $\XSCF(a):=\X(a)|\xF$. 
%
As a consequence, it no longer makes sense to consider a hard  constraint of the form $h(\xSCF(a))>0.5$, i.e., that the prediction needs to change. %
Instead, we can reason about the expected classifier output under the counterfactual distribution, leading to the following \textit{probabilistic version of the individualised recourse optimisation problem}: 
\begin{equation}
\label{eq:optimisation_problem_probabilistic_recourse}
\textstyle
\min_{a=do(\X_\I=\th)\in\AF}\quad \costF(a) \quad \subto \quad  \E_{\XSCF(a)}\left[h\left(\XSCF(a)\right)\right] \geq \texttt{thresh}(a).
\end{equation}
Note that the threshold $\texttt{thresh}(a)$ is allowed to depend on $a$. For example, an intuitive choice is
\begin{equation}
\label{eq:LCB}
\textstyle
\texttt{thresh}(a) = 0.5 +\gamma_\LCB \sqrt{\text{Var}_{\XSCF(a)}\left[h\left(\XSCF(a)\right)\right]}
\end{equation}
which has the  interpretation of the lower-confidence bound  crossing the decision boundary of $0.5$.
Note that larger values of the hyperparameter $\gamma_\LCB$ lead to a more conservative approach to recourse, while for $\gamma_\LCB=0$ merely crossing the decision boundary with $\geq 50\%$ chance suffices. 


\vspace{\negspace}
\section{Subpopulation-based algorithmic recourse via interventions and \cate s}
\label{sec:subpopulation_based_recourse_via_CATE}
\vspace{\negspace}

\newcommand{\nodesize}{1.5em}
\newcommand{\ys}{2em}
\newcommand{\xs}{-1.75em}
\newcommand{\ang}{35}
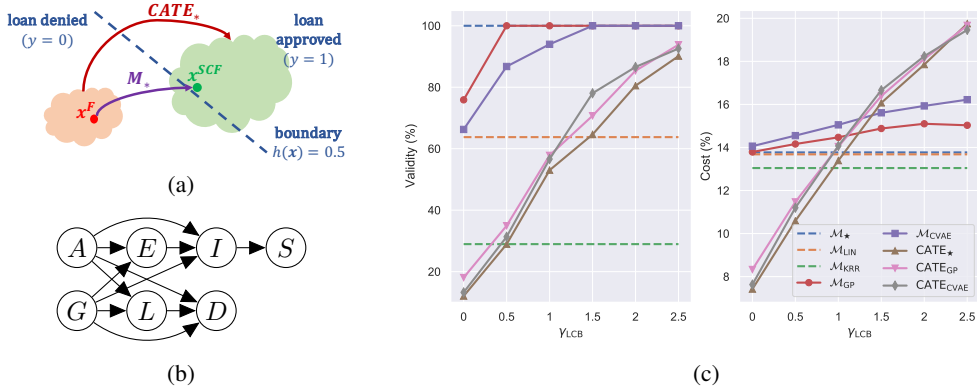
\begin{figure}
    \centering
    \begin{minipage}[b]{0.35\textwidth}
        \centering
        \includegraphics[width=\textwidth]{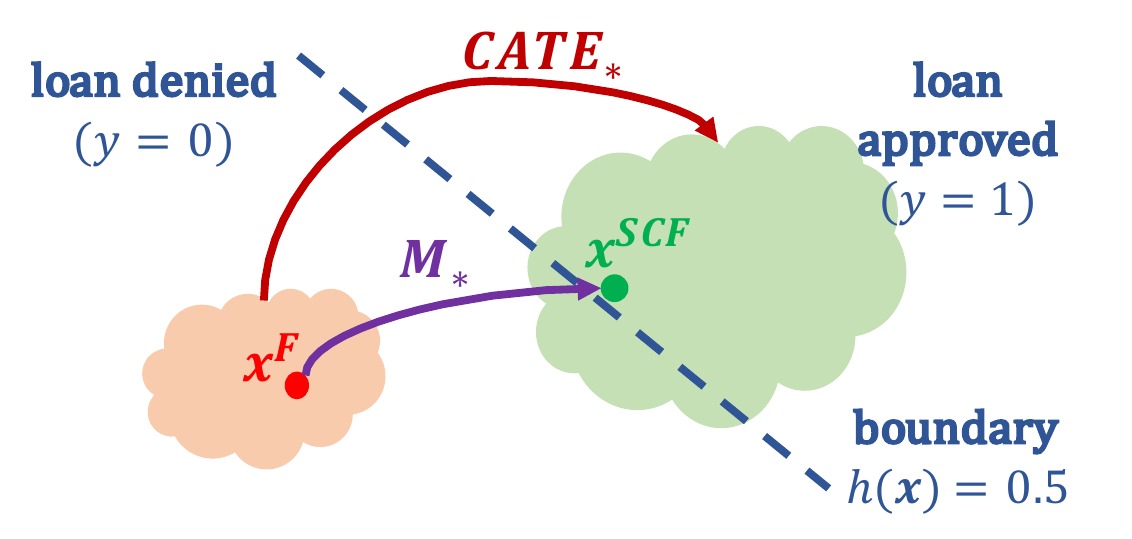}
        \vspace{-1.5em}
        \subcaption{}
        \vspace{-0.75em}
        \label{fig:illustration}
        \begin{tikzpicture}
            \centering
            \node (A) [latent, minimum size=\nodesize] {$A$};
            \node (G) [latent, below=of A, yshift=\ys, minimum size=\nodesize] {$G$};
            \node (E) [latent, right=of A, xshift=\xs, minimum size=\nodesize] {$E$};
            \node (L) [latent, right=of G, xshift=\xs, minimum size=\nodesize] {$L$};
            \node (D) [latent, right=of L, xshift=\xs, minimum size=\nodesize] {$D$};
            \node (I) [latent, right=of E, xshift=\xs, minimum size=\nodesize] {$I$};
            \node (S) [latent, right=of I, xshift=\xs, minimum size=\nodesize] {$S$};
            \edge {A, G} {E, L};
            \edge{A}{D};
            \path[->] (A) edge[bend right=-\ang] node[yshift=1em] {} (I);
            \path[->] (G) edge[bend right=\ang] node[yshift=1em] {} (D); \edge{L}{D};
            \edge{G,E}{I};
            \edge{I}{S};
        \end{tikzpicture}
        \subcaption{}
        \label{fig:german_credit}
    \end{minipage}%
    \begin{minipage}[b]{0.65\textwidth}
    \includegraphics[width=\textwidth]{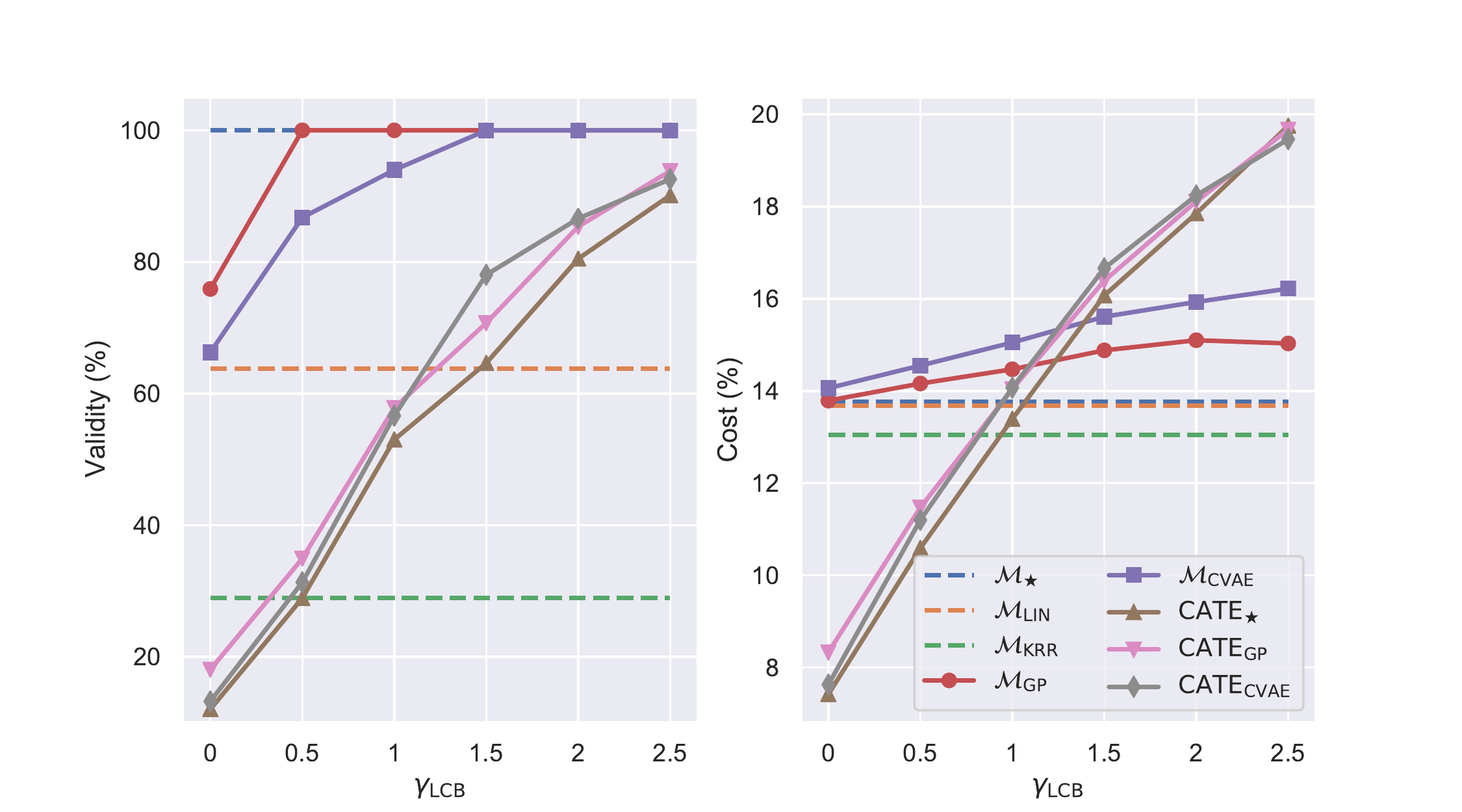}
    \vspace{-1em}
    \subcaption{}
    \label{fig:gamma_LCB}
    \end{minipage}
    \caption{(a) Illustration of point- and subpopulation-based recourse approaches. (b) Assumed causal graph for the semi-synthetic loan approval dataset. (c) Trade-off between validity and cost which can be controlled via $\gamma_\LCB$ for the probabilistic recourse methods.}
    \label{fig:german_credit_graph_and_LCB_plots}
\end{figure}

 The \gpscm{} approach in \cref{sec:recourse_via_probabilistic_counterfactuals} allows us to average over an infinite number of
 (non-)linear
 structural equations, under the assumption of additive Gaussian noise. However, this assumption may still not hold under the true \scm, leading to  sub-optimal or inefficient solutions to the recourse problem. 
Next, we remove any assumptions about the structural equations, and propose a second approach that does not aim to approximate an individualised counterfactual distribution, but instead considers the effect of interventions on a subpopulation defined by certain shared characteristics with the given (factual) individual $\xF$.
The key idea behind this approach resembles the notion of conditional average treatment effects (\cate) \citep{abrevaya2015estimating} (illustrated in \fig\ref{fig:illustration}) and is based on the fact that any intervention  $do(\X_\I=\th)$ only influences the descendants $\d(\I)$ of the intervened-upon variables, while the non-descendants $\nd(\I)$ remain unaffected. 
Thus, when evaluating an intervention, we can condition on $\X_{\nd(\I)}=\xF_{\nd(\I)}$, thus selecting a subpopulation of individuals similar to the factual subject. 


Specifically, we propose to solve the following\textit{ subpopulation-based recourse optimisation problem}
\begin{equation}
\label{eq:optimisation_problem_CATE}
\min_{a
\in\AF}\quad  \costF(a) \quad \subto \quad  \E_{\X_{\d(\I)}|do(\X_\I = \th), \xF_{\nd(\I)}}\big[h\big(\xF_{\nd(\I)}, \th, \X_{\d(\I)}\big)\big]  \geq \texttt{thresh}(a),
\end{equation}
where, in contrast to \eqref{eq:optimisation_problem_probabilistic_recourse}, the expectation is taken over the corresponding interventional distribution.

In general, this interventional distribution does not match the conditional distribution, i.e.,  $P_{\X_{\d(\I)}\given do(\X_\I=\th), \xF_{\nd(\I)}}\neq P_{\X_{\d(\I)}\given \X_\I=\th, \xF_{\nd(\I)}}$,
because some spurious correlations in the observational distribution do not transfer to the interventional setting.
For example, in \fig \ref{fig:causal_graph} we have that $P_{X_2\given do(X_1=x_1,X_3=x_3)}=P_{X_2\given X_1=x_1}\neq P_{X_2\given X_1=x_1,X_3=x_3}$. 
Fortunately, the interventional distribution can still be identified from the observational one, as stated
in the following proposition.


\begin{restatable}[]{proposition}{IDIVDIST}
\label{prop:identifiability_of_interventional_dist}
Subject to causal sufficiency, $P_{\X_{\d(\I)}\given do(\X_\I=\th), \xF_{\nd(\I)}}$ is observationally identifiable:
\begin{equation}
\label{eq:identifiability_of_interventional_dist}
\textstyle
p\big(\X_{\d(\I)}\given do(\X_\I=\th), \xF_{\nd(\I)}\big)=  \left. \prod_{r\in\d(\I)}  p\left(X_{r} \given \X_{pa(r)}\right) \right|_{\X_\I=\th, \X_{\nd(\I)}=\xF_{\nd(\I)}}.
\end{equation}
\end{restatable}

As evident from Proposition \ref{prop:identifiability_of_interventional_dist}, tackling the optimisation problem in \eqref{eq:optimisation_problem_CATE} in the general case (i.e., for arbitrary graphs and intervention sets $\I$) requires estimating the stable conditionals $P_{X_r\given \X_{\pa(r)}}$ (a.k.a.\ causal Markov kernels) in order to compute the interventional expectation via \eqref{eq:identifiability_of_interventional_dist}. 
For convenience (see \cref{sec:optimisation_problem} for details), here we opt for latent-variable implicit density models, but other conditional density estimation approaches may be also be used~\citep[e.g.,][]{bashtannyk2001bandwidth, bishop1994mixture, trippe2018conditional}.
Specifically, we model each conditional $p(x_r\given \x_{\pa(r)})$ with a conditional variational autoencoder (\cvae) \citep[][]{sohn2015learning}
as: 
\begin{equation}
    \label{eq:CVAE_model}
    \textstyle
    p(x_r\given \x_{\pa(r)}) \approx p_{\psi_r}(x_r\given \x_{\pa(r)}) = 
    \medint \int p_{\psi_r}(x_r\given \x_{\pa(r)}, \z_r) p(\z_r) d\z_r, \quad \quad  p(\z_r):=\N(\0, \Id).
\end{equation}
To facilitate sampling $x_r$ (and in analogy to the deterministic mechanisms $f_r$ in \scm s), we opt for deterministic decoders in the form of neural nets $\dec_r$ parametrised by $\psi_r$, i.e., $p_{\psi_r}(x_r\given \x_{\pa(r)}, \z_r) := \delta\left(x_r - \dec_r(\x_{\pa(r)}, \z_r; \psi_r)\right)$, 
and rely on variational inference \cite{wainwright2008graphical}, amortised with approximate posteriors $q_{\phi_r}(\z_r \given x_r,  \x_{\pa(r)})$ parametrised by encoders in the form of neural nets with parameters $\phi_r$. We learn both the encoder and decoder parameters by maximising the evidence lower bound (ELBO) using stochastic gradient descend \citep{bottou2008tradeoffs, kingma2014adam, kingma2013auto, rezende2014stochastic}. 
For further details, we refer to Appendix~\ref{app:cvae_training_details}.

\begin{remark}
\label{remark:pseudo_CF}
The collection of \cvae s can be interpreted as learning an approximate \scm\ of the form
\begin{equation}
\label{eq:CVAE-SCM}
    \MH_\cvae:
    \quad \quad
    \SH=\{X_r := \dec_r(\X_{\pa(r)}, \z_r;\psi_r)\}_{r=1}^d, 
    \quad \quad 
    \z_r\sim \N(\0, \Id)\quad \forall r\in[d]
\end{equation}
However, this family of \scm s may not allow to identify the true \scm\ (provided it can be expressed as above) from data without additional assumptions.
Moreover, exact posterior inference over $\z_r$ given $\xF$ is intractable, and we need to resort to approximations 
instead.
It is thus unclear whether sampling from $q_{\phi_r}(\z_r\given x^{\texttt{F}}_r, \xF_{\pa(r)})$ instead of from $p(\z_r)$ in \eqref{eq:CVAE_model} can be interpreted as a counterfactual within \eqref{eq:CVAE-SCM}.
For further discussion on such ``pseudo-counterfactuals'' we refer to Appendix~\ref{app:nonidentifiability}.
\end{remark}

\vspace{\negspace}
\section{Solving the probabilistic-recourse optimisation problems}
\label{sec:optimisation_problem}
\vspace{\negspace}
We now discuss how to solve the resulting optimisation problems in \eqref{eq:optimisation_problem_probabilistic_recourse} and \eqref{eq:optimisation_problem_CATE}.
First, note that both problems differ only on the distribution over which the expectation in the constraint is taken: in \eqref{eq:optimisation_problem_probabilistic_recourse} this is the counterfactual distribution of the descendants given in Proposition \ref{prop:counterfactual_distribution}; and in \eqref{eq:optimisation_problem_CATE} it is the interventional distribution identified in Proposition \ref{prop:identifiability_of_interventional_dist}. 
In either case, computing the expectation for an arbitrary classifier $h$ is intractable. Here, we approximate these integrals via Monte Carlo by sampling $\x_{\d(\I)}^{(m)}$ from the interventional or counterfactual distributions resulting from $a=do(\X_\I=\th)$, i.e.,
\begin{equation*}
\vspace{-0.1em}
\textstyle \E_{\X_{\d(\I)\given \th}},\big[h\big(\xF_{\nd(\I)}, \th, \X_{\d(\I)}\big)\big]\approx \frac{1}{M}
    \sum_{m=1}^M h\big(\xF_{\nd(\I)}, \th, \x_{\d(\I)}^{(m)}\big).  
\end{equation*}

\paragraph{Brute-force approach.  }
 A way to solve \eqref{eq:optimisation_problem_probabilistic_recourse} and \eqref{eq:optimisation_problem_CATE} is to (i) iterate over $a\in\AF$, with $\AF$ being a finite set of feasible actions (possibly as a result of discretising in the case of a continuous search space); (ii) approximately evaluate the constraint via Monte Carlo
 ; and (iii) select a minimum cost action amongst all evaluated candidates satisfying the constraint.
However, this may be computationally prohibitive and  yield suboptimal interventions due to discretisation. 

\paragraph{Gradient-based approach. }
 Recall that, for actions of the form $a=do(\X_\I=\th)$,
we need to optimise over both the intervention \textit{targets} $\I$ and the intervention \textit{values} $\th$.
Selecting targets is a hard combinatorial optimisation problem, 
as there are $2^{d'}$ possible choices for $d'\leq d$ actionable features, with a potentially infinite number of intervention values.
We therefore consider different choices of targets $\I$ in parallel, and propose a gradient-based approach suitable for differentiable classifiers
to efficiently find an optimal $\th$  for a given intervention set $\I$.
\footnote{For large $d$ when enumerating all $\I$ becomes computationally prohibitive, we can upper-bound the allowed number of variables to be intervened on simultaneously (e.g., $|\I|\leq 3$), or choose a greedy approach to select $\I$.}
In particular, we first rewrite the constrained optimisation problem in unconstrained form with Lagrangian~\cite{karush1939minima, kuhn1951nonlinear}:  
\begin{equation}
    \label{eq:optimisation_problem_proba_recourse_Lagrange}
\mathcal{L}(\th,\lambda):=\costF(a) + \lambda
\big(\texttt{thresh}(a) - 
\E_{\X_{\d(\I)\given \th}}\big[h\big(\xF_{\nd(\I)}, \th, \X_{\d(\I)}\big)\big]
\big). 
\end{equation}
We then solve the saddle point problem $\min_{\th} \max_\lambda \mathcal{L}(\th,\lambda)$ arising from~\eqref{eq:optimisation_problem_proba_recourse_Lagrange}
with stochastic gradient descent~\cite{bottou2008tradeoffs, kingma2014adam}.
Since both the  \gpscm\ counterfactual~\eqref{eq:GP_SCM_counterfactual_dis} and the \cvae\ interventional distributions~\eqref{eq:CVAE_model} admit a reparametrisation trick   \citep{kingma2013auto, rezende2014stochastic}, we can differentiate through the constraint:  
\begin{equation}
\label{eq:reparametrisation_trick}
    \nabla_\th \E_{\X_{\d(\I)}}\big[h\big(\xF_{\nd(\I)}, \th, \X_{\d(\I)}\big)\big]
    =
     \E_{\z\sim\N(\0,\Id)}\big[\nabla_\th h\big(\xF_{\nd(\I)}, \th, \x_{\d(\I)}(\z)\big)\big].
\end{equation}
Here, $\x_{\d(\I)}(\z)$ is obtained by iteratively computing all descendants in topological order: either substituting $\z$ together with the other parents into the decoders $\dec_r$
for the \cvae s, or by using the Gaussian reparametrisation $x_{r}(\z)=\mu+\sigma \z$ with $\mu$ and $\sigma$ given by \eqref{eq:GP_SCM_counterfactual_dis} for the \gpscm.
A similar gradient estimator for the variance which enters $\texttt{thresh}(a)$ for $\gamma_{\LCB}\neq 0$ is derived in Appendix~\ref{app:derivation_of_variance_grad}.

\vspace{\negspace}
\section{Experimental results}
\label{sec:experiments}
\vspace{\negspace}

In our experiments, we compare different approaches for \textit{causal} algorithmic recourse on synthetic and semi-synthetic data sets.
Additional results can be found in Apendix~\ref{app:additional_results}.

\paragraph{Compared methods.}
We compare the naive point-based recourse approaches $\MH_\lin$ and $\MH_\KR$ mentioned at the beginning of \cref{sec:recourse_via_probabilistic_counterfactuals} as baselines with the proposed counterfactual \gpscm\ $\MH_\gp$ and the \cvae\ approach for sub-population-based recourse ($\cate_\cvae$).
For completeness, we also consider a $\cate_\gp$ approach as a \gp\ can also be seen as modelling each conditional as a Gaussian,\footnote{Sampling from the noise prior instead of the posterior in~\eqref{eq:noise_posterior} leads to an interventional distribution in~\eqref{eq:GP_SCM_counterfactual_dis}.}
and also evaluate the ``pseudo-counterfactual'' $\MH_\cvae$ approach discussed in Remark \ref{remark:pseudo_CF}. 
Finally, we report oracle performance for individualised $\Mstar$ and sub-population-based recourse methods $\CATEstar$ by sampling counterfactuals and interventions from the true underlying \scm.
We note that a comparison with non-causal recourse approaches that assume independent features~\cite[][]{ustun2019actionable, sharma2020certifai} or consider causal relations to generate counterfactual explanations but not recourse actions~\cite{joshi2019towards, mahajan2019preserving} is neither natural nor straight-forward, because it is unclear whether descendant variables should be allowed to change, whether keeping their value constant should incur a cost, and, if so, how much, c.f.~\cite{karimi2020algorithmic}.

\paragraph{Metrics. }  We compare recourse actions recommended by the different methods in terms of  \textit{cost}, computed as the  L2-norm between the intervention $\th_\I$ and the factual value $\xF_\I$, normalised by the range of each feature $r\in\I$ observed in the training data;  and \textit{validity}, computed as the percentage of individuals for which the recommended actions result in a favourable prediction under the true (oracle) \scm.
For our probabilistic recourse methods, we also report the lower confidence bound $\LCB:=\E[h]-\gamma_\LCB\sqrt{\text{Var}[h]}$ of the selected action under the given method.

\begin{table*}[t]
  \small
   \setlength{\tabcolsep}{2.5pt}
  \caption{Experimental results for the gradient-based approach on different 3-variable \scm s. We show average performance $\pm1$ standard deviation for $N_\text{runs} = 100$, $N_\text{MC-samples}=100$, and $\gamma_\LCB=2$.
    }
  \label{table:3-vars-main}
  \begin{tabular}{l c c c c c c c c c}
    \toprule
    \multirow{2}{*}{Method} &
    \multicolumn{3}{c}{\textsc{linear} \scm} &
    \multicolumn{3}{c}{\textsc{non-linear} \anm} &
    \multicolumn{3}{c}{\textsc{non-additive} \scm} \\
    \cmidrule(r){2-4} \cmidrule(r){5-7} \cmidrule(r) {8-10}
                     & $\Validstar$ (\%)  &  $\LCB$  & Cost (\%)   & $\Validstar$ (\%)  &  $\LCB$  & Cost (\%)  & $\Validstar$ (\%)  &  $\LCB$  & Cost (\%) \\
    \midrule
    
    $\Mstar$         & 100  & -           & 10.9$\pm$7.9           & 100  & -           & 20.1$\pm$12.3         & 100  & -           & 13.2$\pm$11.0 \\
    $\MH_{\lin}$     & 100  & -           & 11.0$\pm$7.0           &  54  & -           & 20.6$\pm$11.0         &  98  & -           & 14.0$\pm$13.5 \\
    $\MH_\KR$        &  90  & -           & 10.7$\pm$6.5           &  91  & -           & 20.6$\pm$12.5         &  70  & -           & 13.2$\pm$11.6 \\
    $\MH_{\gp}$      & 100  & .55$\pm$.04 & 12.2$\pm$8.3           & 100  & .54$\pm$.03 & 21.9$\pm$12.9         &  95  & .52$\pm$.04 & 13.4$\pm$12.8 \\
    $\MH_{\cvae}$    & 100  & .55$\pm$.07 & 11.8$\pm$7.7           &  97  & .54$\pm$.05 & 22.6$\pm$12.3         &  95  & .51$\pm$.01 & 13.4$\pm$12.2 \\
    $\CATEstar$      &  90  & .56$\pm$.07 & 11.9$\pm$9.2           &  97  & .55$\pm$.05 & 26.3$\pm$21.4         & 100  & .52$\pm$.02 & 13.5$\pm$13.0 \\
    $\cate_\gp$      &  93  & .56$\pm$.05 & 12.2$\pm$8.4           &  94  & .55$\pm$.06 & 25.0$\pm$14.8         &  94  & .52$\pm$.03 & 13.2$\pm$13.1 \\
    $\cate_\cvae$    &  89  & .56$\pm$.08 & 12.1$\pm$8.9           &  98  & .54$\pm$.05 & 26.0$\pm$14.3         & 100  & .52$\pm$.05 & 13.6$\pm$12.9 \\
    \bottomrule
  \end{tabular}
  \vspace{-0.5em}
\end{table*}

\paragraph{Synthetic 3-variable \scm s under different assumptions. }
In our first set of experiments, we consider three classes of \scm s over three variables with the same causal graph as in \fig\ref{fig:causal_graph}.
To test robustness of the different methods to assumptions about the form of the true structural equations, we consider a linear \scm, a non-linear \anm, and a more general, multi-modal \scm\ with non-additive noise.
For further details on the exact form we refer to Appendix~\ref{app:experimental_details}.

Results are shown in Table~\ref{table:3-vars-main}.
We observe that the point-based recourse approaches 
perform (relatively) well in terms of both validity and cost, when their underlying assumptions are met (i.e., $\MH_\lin$ on the linear \scm\ and $\MH_\KR$ on the nonlinear \anm).
Otherwise, validity significantly drops as expected (see, e.g., the results of $\MH_\lin$ on the non-linear $\anm$, or of $\MH_\KR$ on the non-additive \scm).
Moreover, we note that the inferior performance of $\MH_\KR$ compared to $\MH_\lin$ on the linear \scm\ suggests an overfitting problem, which does not occur for its more conservative probabilistic counterpart $\MH_\gp$.
Generally, the individualised approaches $\MH_\gp$ and $\MH_\cvae$ perform very competitively in terms of cost and validity, especially on the linear and nonlinear \anm s.
The subpopulation-based $\cate$ approaches on the other hand, perform particularly well on the challenging non-additive \scm\ (on which the assumptions of \gp\ approaches are violated) where $\cate_\cvae$ achieves perfect validity as the only non-oracle method.
%
%
As expected, the subpopulation-based approaches generally lead to higher cost than the individualised ones, since the latter only aim to achieve recourse  only for a given individual while the former do it for an entire group
(see Fig.~\ref{fig:illustration}).

\paragraph{Semi-synthetic 7-variable \scm\ for loan-approval.  }
We also test our methods on a larger semi-synthetic \scm\ inspired by the German Credit UCI dataset~\cite{murphy1994uci}.
We consider the variables age $A$, gender $G$, education-level $E$, loan amount $L$, duration $D$, income $I$, and savings $S$ with causal graph shown in \fig\ref{fig:german_credit}.
We model age $A$, gender $G$ and loan duration $D$ as non-actionable variables, but consider $D$ to be mutable, i.e., it cannot be manipulated directly but is allowed to change (e.g., as a consequence of an intervention on $L$).
The \scm\ includes linear and non-linear relationships, as well as different types of variables and noise distributions, and is described in more detail in Appendix~\ref{app:experimental_details}.

The results are summarised in Table~\ref{table:german-credit}, where we observe that the insights discussed above similarly apply for data generated from a more complex \scm, and for different classifiers. 
Finally, we show the influence of $\gamma_\LCB$ on the performance of the proposed probabilistic approaches in \fig\ref{fig:gamma_LCB}.  
We observe that lower values of $\gamma_\LCB$ lead to lower validity (and cost), especially for the $\cate$ approaches.  As $\gamma_\LCB$ increases validity approaches the corresponding oracles $\Mstar$ and $\CATEstar$, outperforming the point-based recourse approaches. 
In summary, our probabilistic recourse approaches are not only more robust, but also allow controlling the trade-off between validity and cost using $\gamma_\LCB$.

\begin{table*}[t]
  \small
  \setlength{\tabcolsep}{2pt}
  \caption{Experimental results for the 7-variable \scm\ for loan-approval. We show average performance $\pm1$ standard deviation for $N_\text{runs} = 100$, $N_\text{MC-samples}=100$, and $\gamma_\LCB=2.5$. For linear and non-linear logistic regression as classifiers, we use the gradient-based approach, whereas for the non-differentiable random forest classifier we rely on the brute-force approach (with 10 discretised bins per dimension) to solve the recourse optimisation problems.
    }
  \label{table:german-credit}
  \begin{tabular}{l c c c c c c c c c}
    \toprule
    \multirow{2}{*}{Method} &
    \multicolumn{3}{c}{\textsc{linear log.\ regr.}} &
    \multicolumn{3}{c}{\textsc{non-lin.\ log.\ regr.\ (mlp)}} &
    \multicolumn{3}{c}{\textsc{random forest(brute-force)}} \\
    \cmidrule(r){2-4} \cmidrule(r){5-7} \cmidrule(r){8-10}
                  & $\Validstar$ (\%)   &  $\LCB$   & Cost (\%)        & $\Validstar$ (\%)  &  $\LCB$   & Cost (\%)       & $\Validstar$ (\%) &  $\LCB$   & Cost (\%) \\
    \midrule
    
    $\Mstar$      & 100  & -           & 15.8$\pm$ 7.6         & 100  & -           & 11.0$\pm$7.0         & 100  & -           & 15.2$\pm$7.5 \\
    $\MH_{\lin}$  &  19  & -           & 15.4$\pm$ 7.4         &  80  & -           & 11.0$\pm$6.9         &  94  & -           & 15.6$\pm$7.6 \\
    $\MH_\KR$     &  41  & -           & 15.6$\pm$ 7.5         &  87  & -           & 11.1$\pm$7.0         &  92  & -           & 15.1$\pm$7.4 \\
    $\MH_{\gp}$   & 100  & .50$\pm$.00 & 18.0$\pm$ 7.7         & 100  & .52$\pm$.04 & 11.7$\pm$7.3         & 100  & .66$\pm$.14 & 16.3$\pm$7.4 \\
    $\MH_{\cvae}$ & 100  & .50$\pm$.00 & 16.6$\pm$ 7.6         &  99  & .51$\pm$.01 & 11.3$\pm$6.9         & 100  & .66$\pm$.14 & 15.9$\pm$7.4 \\
    $\CATEstar$   &  93  & .50$\pm$.01 & 22.0$\pm$ 9.4         &  95  & .52$\pm$.05 & 12.0$\pm$7.7         &  98  & .66$\pm$.15 & 17.0$\pm$7.3 \\
    $\cate_\gp$   &  93  & .50$\pm$.02 & 21.7$\pm$ 9.2         &  93  & .51$\pm$.06 & 12.0$\pm$7.4         & 100  & .67$\pm$.15 & 17.1$\pm$7.4 \\
    $\cate_\cvae$ &  94  & .49$\pm$.01 & 23.7$\pm$11.3         &  95  & .51$\pm$.03 & 12.0$\pm$7.8         & 100  & .68$\pm$.15 & 17.9$\pm$7.4 \\
    \bottomrule
  \end{tabular}
    \vspace{-.5em}
\end{table*}

\vspace{\negspace}
\section{Discussion}
\label{sec:discussion}
\vspace{\negspace}

\paragraph{Assumptions, limitations, and extensions.} Throughout the paper, we have assumed a known causal graph and causal sufficiency. 
While this may not hold for all settings, it is the minimal necessary set of assumptions for causal reasoning from observational data alone. 
%
Access to instrumental variables or experimental data may help further relax these assumptions~\cite{angrist1996identification,cooper1999causal, tian2001causal}.
Moreover, if only a partial graph is available or some relations are known to be confounded, one will need to restrict recourse actions to the subset of interventions that are still identifiable~\cite{shpitser2006identification, shpitser2008complete, tian2002general}.
An alternative approach could address causal sufficiency violations by relying on latent variable models to estimate confounders from  multiple causes~\cite{wang2019blessings} or proxy variables~\cite{louizos2017causal}, or to work with bounds on causal effects instead~\citep{balke1994counterfactual, tian2000probabilities}.
We relegate the investigation of these settings to future work.

\paragraph{On the counterfactual vs interventional nature of recourse.}
Given that we address two different notions of recourse---counterfactual/individualised (rung 3) vs.\ interventional/subpopulation-based (rung 2)---one may ask which framing is more appropriate.
Since the main difference is whether the background variables $\U$ are assumed fixed (counterfactual) or not (interventional) when reasoning about actions, we believe that this question is best addressed by thinking about the type of environment and interpretation of $\U$:
if the environment is static, or if $\U$ (mostly) captures unobserved information about the individual, the counterfactual notion seems to be the right one;
if, on the other hand, $\U$ also captures environmental factors which may change, e.g., between consecutive loan applications, then the interventional notion of recourse may be more appropriate.
In practice, both notions may be present (for different variables), and the proposed approaches can be combined depending on the available domain knowledge since each parent-child causal relation is treated separately.
We emphasise that the subpopulation-based approach is also practically motivated by a reluctance to make (parametric) assumptions about the structural equations which are untestable but necessary for counterfactual reasoning.
It may therefore be useful to avoid problems of misspecification, even for counterfactual recourse, as demonstrated experimentally for the non-additive \scm.










\section{Conclusion}
\label{sec:conclusion}
In this work, we studied the problem of algorithmic recourse from a causal perspective. As negative result, we first showed that algorithmic recourse cannot be guaranteed in the absence of perfect knowledge about the underlying $\scm$ governing the world, which unfortunately is not available in practice. 
To address this limitation, we proposed two probabilistic approaches to achieve  recourse under more realistic assumptions. In particular, we  derived i) an individual-level recourse approach based on \gp{s} that approximates the counterfactual distribution by averaging over the family of additive Gaussian \scm{s}; and ii) a subpopulation-based approach, which assumes that only the causal graph is known and  makes use of \cvae{s} to estimate the conditional average treatment effect of an intervention on a subpopulation similar to the individual seeking recourse. 
Our experiments showed that the proposed probabilistic approaches not only result in more robust recourse interventions than approaches based on point estimates of the \scm, but also allows to trade-off validity and cost.   

\clearpage
\section*{Broader Impact}
Our work falls into the domain of explainable AI, which---given the increasing use of often intransparent (``blackbox'') machine learning models in consequential decision making---is of rapidly-growing societal importance.
In particular, we consider the task of enabling and facilitating \textit{algorithmic recourse}, which aims to provide individuals with guidance and recommendations on how best (i.e., efficiently and ideally at low cost) to recover from unfavourable decisions made by an automated system.
To address this task, we build on the framework of causal modelling, which constitutes a principled and mathematically rigorous way to reason about the downstream effects of actions. 
Since correlation does not imply causation, this requires to make additional assumptions based on a general understanding of the domain at hand.
While this may perhaps seem restrictive at first, we point out that other approaches to explainability also make implicit assumptions of a causal nature (e.g., that all features can be changed at will without affecting others in the case of ``counterfactual'' explanations), without explicitly and clearly stating such assumptions.
The advantage of phrasing assumptions about relations between features in the form of a causal graph is that the latter is transparent and intuitive to understand and can thus be challenged by decision makers and individuals alike. 

While theoretically sound from a causal perspective, at the same time, our method is aimed at being practical by not making further assumptions beyond the causal graph which would be hard or impossible to test or challenge empirically---in contrast to the assumed known specification of the full \scm\ in~\cite{karimi2020algorithmic}.
We start from the position that the model is only partially known, and use this to motivate probabilistic approaches to causal algorithmic recourse which take uncertainty into account.
Our approaches are more robust to misspeficiation than naive point-based recourse methods (as demonstrated experimentally): ``system-failure'' is thus fundamentally baked in to our methods.
Moreover, the interpretable ``conservativeness parameter'' $\gamma_\LCB$ can be used trade-off the desired level of robustness against the effort an individual is willing to put into achieving recourse.

The importance of causal reasoning for an ethical and socially beneficial use of ML-assisted technology has also been stressed in a number of recent works in the field of explainability and fair algorithmic decision making~\cite{kusner2017counterfactual, russell2017worlds, kilbertus2017avoiding, zhang2018equality, zhang2018fairness, chiappa2019path, von2020fairness, gupta2019equalizing}.
We thus hope that some of the probabilistic approaches for causal reasoning under imperfect knowledge proposed in this work may also prove useful for related tasks such as fairness, accountability, transparency.
To this end, we have created a user-friendly implementation of all the approaches proposed in this work that we will make publicly available to be scrutinised, re-used, and further improved by the community.
The code is highly flexible and only requires the specification of 
a causal graph, as well as a labelled training dataset.

Since our work considers the classifier as given, it is possible that it is explicitly discriminatory or reproduces biases in the data.
While not directly addressing this problem, our work aims to enable individuals to overcome a potentially unfairly obtained decision with minimal effort.
If successful recourse examples are included in future training data, this may help de-bias a system over time; we consider the intersection of our work with fair decision making in the context of a classifier evolving over time as the result of further data collection~\cite{kilbertus2019fair} a fruitful and important direction for future research.
In addition, observing that certain minority groups consistently receive more costly recourse recommendations may be a way to reveal bias in the underlying decision making system.

While our framework is intended to help individuals increase their chances for a more favourable prediction given that they were, e.g., denied a loan or bail, we cannot rule out a priori, that the same approach could also be used by foes in unintended ways, e.g., to ``game'' a spam filter or similar system built to protect society from harm.
However, since our framework requires the specification of a causal graph which usually requires an understanding of the domain and the causal influences at play, it is unlikely that it could be abused by a purely virtual system without a human in the loop.

\begin{ack}
The authors would like to thank Adrian Weller, Floyd Kretschmar, Junhyung Park, Matthias Bauer, Miriam Rateike, Nicolo Ruggeri, Umang Bhatt, and Vidhi Lalchand for helpful feedback and discussions.
Moreover, a special thanks to Adrià Garriga-Alonso for insightful input on some of the \gp-derivations and to Adrián Javaloy Bornás for invaluable help with the \cvae-training. AHK acknowledges  NSERC and CLS for generous funding support.
\end{ack}

\small
\bibliographystyle{plainnat}
\bibliography{references}

\newpage
\appendix
\numberwithin{equation}{section}

\section{Proofs}
\label{app:proofs}
\subsection{Proof of Proposition \ref{prop:noise_posterior}}

\GPSCMNP*

\begin{proof}
First, note that, by definition, $\u_r$ is independent of $\f_r=(f_r(\x_{\pa(r)}^1), ..., f_r(\x_{\pa(r)}^n))$ given $\X_{\pa(r)}$.
Moreover, it follows from the assumed GP-SCM model in  \eqref{eq:GP_SCM} and Definition \ref{def:GPSCM}, as well as properties of the GP prior, that  both are multivariate Gaussian random variables with distributions given by
\begin{align}
    \u_r &\sim \N(\0, \sigma_r^2\Id ) \quad \text{independently of}\quad \X_{pa(r)}, \quad \text{and}\\
    \f_r\given\X_{pa(r)} &\sim \N(\0, \K),
\end{align}
where $\0$ denotes the zero vector (or matrix, see below) and $\K$ is as defined in Proposition \ref{prop:noise_posterior}.

Since independent multivariate Gaussian random variables are jointly multivariate Gaussian, we thus have
\begin{equation}
\label{eq:proof_multivariate_normal}
\begin{pmatrix}
 \u_r\\
\f_r
\end{pmatrix}
\given \X_{\pa(r)} \sim \N(\0, \Sigma), \quad \text{where} \quad \Sigma=
\begin{pmatrix}
\sigma_r^2\Id & \0\\
\0 & \K
\end{pmatrix}
\end{equation}
Noting that $\x_r=\f_r + \u_r$ and applying a linear transformation to \eqref{eq:proof_multivariate_normal}, we then obtain
\begin{equation}
\label{eq:proof_multivariate_normal_y}
\begin{pmatrix}
 \u_r\\
\x_r
\end{pmatrix}
\given \X_{\pa(r)}
=
\begin{pmatrix}
 \Id & \0 \\
\Id & \Id\\
\end{pmatrix}
\begin{pmatrix}
 \u_r\\
\f_r
\end{pmatrix}
\given \X_{\pa(r)}
\sim \N(\0, \Tilde{\Sigma}), \quad \text{where} \quad \Tilde{\Sigma}=
\begin{pmatrix}
\sigma_r^2\Id & \sigma_r^2\Id\\
\sigma_r^2\Id & \K+\sigma_r^2\Id
\end{pmatrix}.
\end{equation}
Conditioning on $\x_r$ and using the conditioning formula \citep[e.g.,][]{toussaint2011lecture}, the result follows:
\begin{align}
\label{eq:proof_conditional}
 \u_r
\given \X_{pa(r)}, \x_r
&\sim
\N\left(
\0 + \ss_r\Id (\K+\ss_r \Id)^{-1}(\x_r-\0),
\ss_r\Id-\ss_r\Id(\K+\ss_r \Id)^{-1}\ss_r\Id\right)\\
&\sim
\N\left(\ss_r (\K+\ss_r \Id)^{-1}\x_r, \ss_r\left(\Id-\ss_r(\K+\ss_r \Id)^{-1}\right)\right)
\end{align}
\end{proof}

\subsection{Proof of Proposition \ref{prop:counterfactual_distribution}}

\GPSCMCF*

\begin{proof}
We follow the three steps of abduction, action, and prediction for computing counterfactual distributions (see \cref{sec:background_causality} for more details).
Starting from the factual observation $\xF\in\{x^i\}_{i=1}^n$ generated according to
\begin{equation}
\label{eq:app_structural_equation_F}
x^\texttt{F}_r := f_r(\xF_{\pa(r)})+u_r^\texttt{F},
\end{equation}
we first compute the noise posterior (\textit{abduction}).
According to Proposition \ref{prop:noise_posterior} it is given by a marginal of \eqref{eq:noise_posterior}, i.e.,
\begin{equation}
\label{eq:app_abduction_step}
    u_r^\texttt{F}| \X_{\pa(r)}, \x_r \sim\N(\mu_r^F, s_r^\texttt{F})
\end{equation}
where $\mu_r^\texttt{F}$ is given by element $\texttt{F}$ of the mean vector
\begin{equation}
    \bm{\mu}_r = \ss_r (\K+\ss_r \Id)^{-1}\x_r
\end{equation}
and $s_r^\texttt{F}$ is given by element $(\texttt{F}, \texttt{F})$ of the covariance matrix 
\begin{equation}
    S_r = \ss_r\left(\Id-\ss_r(\K+\ss_r \Id)^{-1}\right)
\end{equation}
of the noise posterior given by \eqref{eq:noise_posterior}.

Next, we simulate the hypothetical intervention by updating the structural equation \eqref{eq:app_structural_equation_F} (\textit{action step}),
\begin{equation}
\label{eq:app_prediction}
    x_r^\texttt{F}(\X_{\pa(r)}=\tilde{\x}_{\pa(r)}):=f_r(\tilde{x}_{\pa(r)})+u_r^\texttt{F}.
\end{equation}
The GP predictive posterior at the new input $\tilde{x}_{\pa(r)}$ has distribution \citep[see, e.g.,][]{williams2006gaussian},
\begin{equation}
\label{eq:app_GP_posterior}
f_r(\tilde{x}_{\pa(r)}) | \X_{\pa(r)}, \x_r 
\sim 
\N(
\tilde{\k}^T(\K+\ss_r\Id)^{-1}\x_r,
\tilde{k} - \tilde{\k}^T (\K+\ss_r\Id)^{-1} \tilde{\k}
).
\end{equation}
Substituting \eqref{eq:app_GP_posterior} and \eqref{eq:app_abduction_step} into \eqref{eq:app_prediction} and noting that the sum of two Gaussians is again Gaussian with mean and variance equal to the sums of means and variances of the two individual Gaussians (\textit{prediction step}) completes the proof.
\end{proof}

\subsection{Proof of Proposition \ref{prop:identifiability_of_interventional_dist}}

\IDIVDIST*

\begin{proof}
This is a direct consequence of the properties of causally sufficient (Markovian) causal models, but we include a derivation for completeness.
Recall that $P$ factorises over its underlying causal graph $\G$ as follows,
\begin{equation}
\label{eq:joint_distribution_causal_factorisation}
    p(\X)
    =  \prod_{r\in[d]}p(X_r|\X_{\pa(r)}).
\end{equation}
This joint distribution is transformed by the intervention $do(\X_\I=\th)$ as follows,
\begin{equation}
\label{eq:joint_interventional_distribution}
    P(\X_{-\I}, do(\X_\I=\th))
    = \delta(\X_\I=\th) \prod_{r\in[d]\setminus\I}P(X_r|\X_{\pa(r)}).
\end{equation}
Splitting the non-intervened variables into descendants $\d(\I)$ and non-descendants $\nd(\I)$, and conditioning on the intervened variables $do(\X_\I=\th)$, we obtain
\begin{equation}
    P(\X_{\nd(\I)}, \X_{\d(\I)}|do(\X_\I=\th))
    = \left. \left(\prod_{r\in\nd(\I)\cup \d(\I)}
    P(X_r|\X_{\pa(r)})\right)
    \right|_{\X_\I=\th}.
\end{equation}
As the non-descendants $\X_{\nd(\I)}$ are, by their very definition, not affected by the intervention, we can write
\begin{equation*}
    P(\X_{\nd(\I)}, \X_{\d(\I)}|do(\X_\I=\th))
    = \left. \left( \prod_{r\in\d(\I)}P(X_r|\X_{\pa(r)})\right)
    \right|_{\X_\I=\th} \prod_{r\in \nd(\I)}P(X_r|\X_{\pa(r)}).
\end{equation*}
We can thus condition on a particular value of $\X_{\nd(\I)}$ to  obtain
\begin{equation}
    P\left(\X_{\d(\I)}\given do(\X_\I=\th), \X_{\nd(\I)}
    =\xF_{\nd(\I)}\right)= \left. \left(\prod_{r\in\d(\I)} P(X_{r} \given \X_{pa(r)})\right) \right|_{\X_\I=\th, \X_{\nd(\I)}=\xF_{\nd(\I)}}
\end{equation}
\end{proof}

\clearpage
\section{Additional results}
\label{app:additional_results}
This section presents additional results complementing those from Section~\ref{sec:experiments}. 
Table~\ref{table:3-vars-app-brute-force} presents results that mirror those in Table~\ref{table:3-vars-main}, where the brute-force approach discussed at the beginning of~\cref{sec:optimisation_problem} is used instead of the gradient-based optimisation. Here, each real-valued feature was discretised into 20 bins within the range of its observed values in the training dataset.

\fig~\ref{fig:app_lcb_plots} mirrors the results in \fig\ref{fig:gamma_LCB}, for which a snapshot ($\gamma_\LCB=2.5$) is also provided in Table~\ref{table:german-credit}. Here we show the trade-off between validity and cost by varying the values of $\gamma_\LCB$, using as trained classifiers a non-linear multilayer perceptron (MLP) in (a) and a non-differentiable random forest classifer in (b).  Note that optimisation for the latter can only be done with the brute-force approach.
All these additional results mostly confirm the insights presented in the main body. 

Finally, Table~\ref{table:3-vars-app-grad-descent-frequency} provides a qualitative comparison of the proposed recourse approaches against the oracles and baselines in terms of their selection of intervention targets.
We show empirically, on the three synthetic datasets, 
that \cate\ approaches have more predictable behaviour, as they are less sensitive to model assumptions, and are thus more preferable for the individual seeking recourse under imperfect causal knowledge.

\begin{table*}[h]
  \small
  \setlength{\tabcolsep}{2.5pt}
  \caption{Experimental results for the brute-force (20-bin discretization) approach on different 3-variable \scm s. We show average performance for $N_\text{runs} = 100$, $N_\text{MC-samples}=100$, and $\gamma_\LCB=2$. The relative trends reflect those in Table~\ref{table:3-vars-main}.
    }
  \label{table:3-vars-app-brute-force}
  \begin{tabular}{l c c c c c c c c c}
    \toprule
    \multirow{2}{*}{Method} &
    \multicolumn{3}{c}{\textsc{linear} \scm} &
    \multicolumn{3}{c}{\textsc{non-linear} \anm} &
    \multicolumn{3}{c}{\textsc{non-additive} \scm} \\
    \cmidrule(r){2-4} \cmidrule(r){5-7} \cmidrule(r) {8-10}
                     & $\Validstar$ (\%)  &  $\LCB$  & Cost (\%)   & $\Validstar$ (\%)  &  $\LCB$  & Cost (\%)  & $\Validstar$ (\%)  &  $\LCB$  & Cost (\%) \\
    \midrule
    
    $\Mstar$         & 100  & -           & 11.0$\pm$5.6           & 100  & -           & 20.7$\pm$11.0         & 100 & -           & 15.8$\pm$ 8.9 \\
    $\MH_{\lin}$     & 100  & -           & 11.3$\pm$5.8           &  60  & -           & 19.9$\pm$ 8.9         &  92 & -           & 17.0$\pm$10.4 \\
    $\MH_\KR$        &  95  & -           & 11.2$\pm$5.6           &  88  & -           & 20.5$\pm$10.7         &  47 & -           & 15.8$\pm$10.6 \\
    $\MH_{\gp}$      & 100  & .55$\pm$.04 & 11.6$\pm$5.8           &  99  & .55$\pm$.04 & 21.2$\pm$10.9         &  88 & .58$\pm$.05 & 16.8$\pm$10.3 \\
    $\MH_{\cvae}$    & 100  & .55$\pm$.04 & 11.5$\pm$5.8           &  95  & .55$\pm$.03 & 21.7$\pm$10.7         &  95 & .59$\pm$.07 & 16.9$\pm$10.3 \\
    $\CATEstar$      &  90  & .57$\pm$.07 & 11.0$\pm$5.5           &  95  & .55$\pm$.05 & 22.8$\pm$10.8         &  99 & .57$\pm$.06 & 16.2$\pm$ 8.9 \\
    $\cate_\gp$      &  92  & .56$\pm$.07 & 11.2$\pm$5.5           &  95  & .55$\pm$.04 & 22.8$\pm$10.9         &  85 & .58$\pm$.07 & 16.4$\pm$10.5 \\
    $\cate_\cvae$    &  90  & .57$\pm$.06 & 11.1$\pm$5.4           &  96  & .55$\pm$.03 & 23.0$\pm$10.8         &  94 & .59$\pm$.07 & 16.8$\pm$10.2 \\
    \bottomrule
  \end{tabular}
  \vspace{-1em}
\end{table*}

\begin{figure}[h]
    \centering
    \begin{minipage}[b]{0.5\textwidth}
        \includegraphics[width=\textwidth]{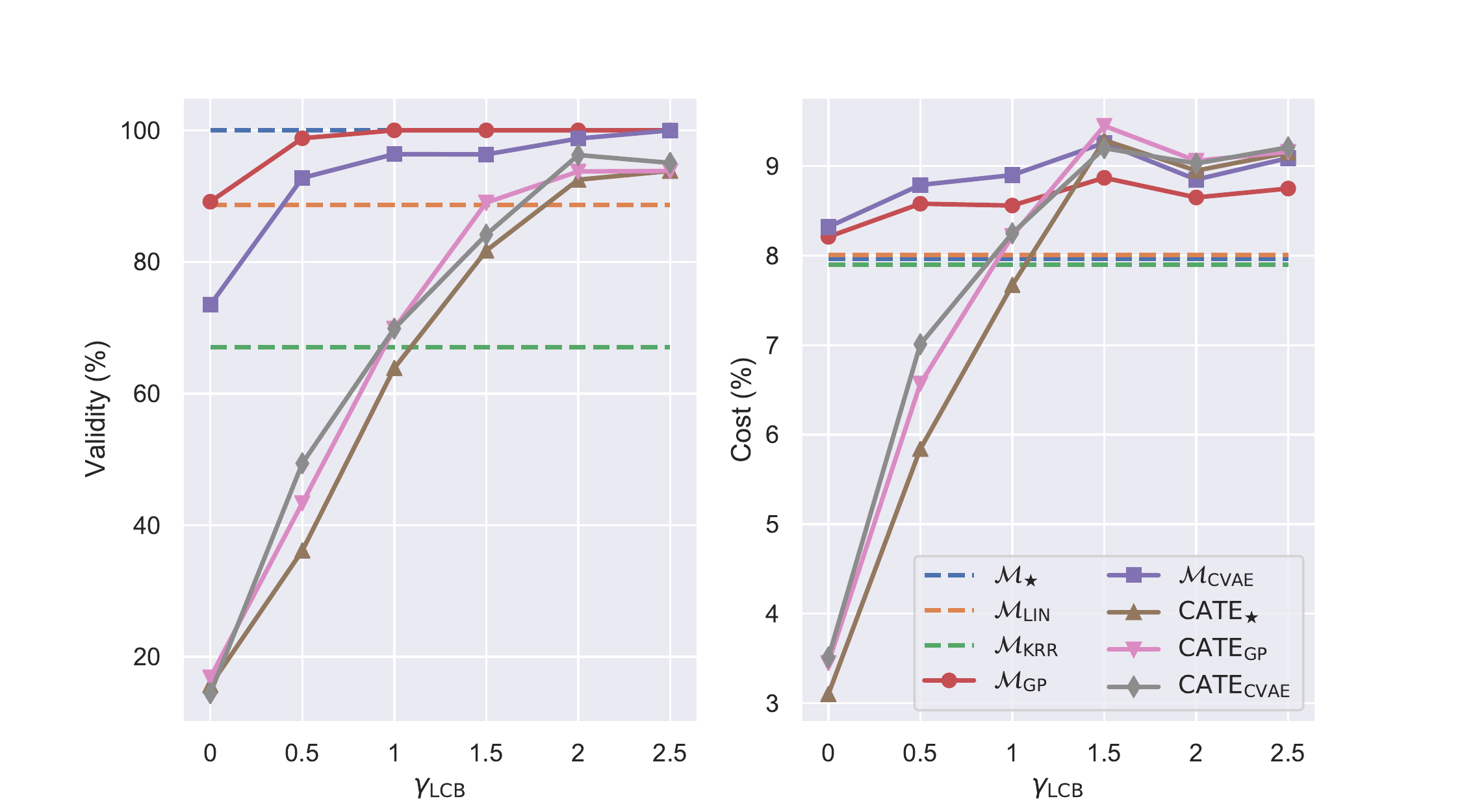}
        \subcaption{MLP}
        \label{fig:gamma_lcb_mlp}
    \end{minipage}%
    \begin{minipage}[b]{0.5\textwidth}
        \includegraphics[width=\textwidth]{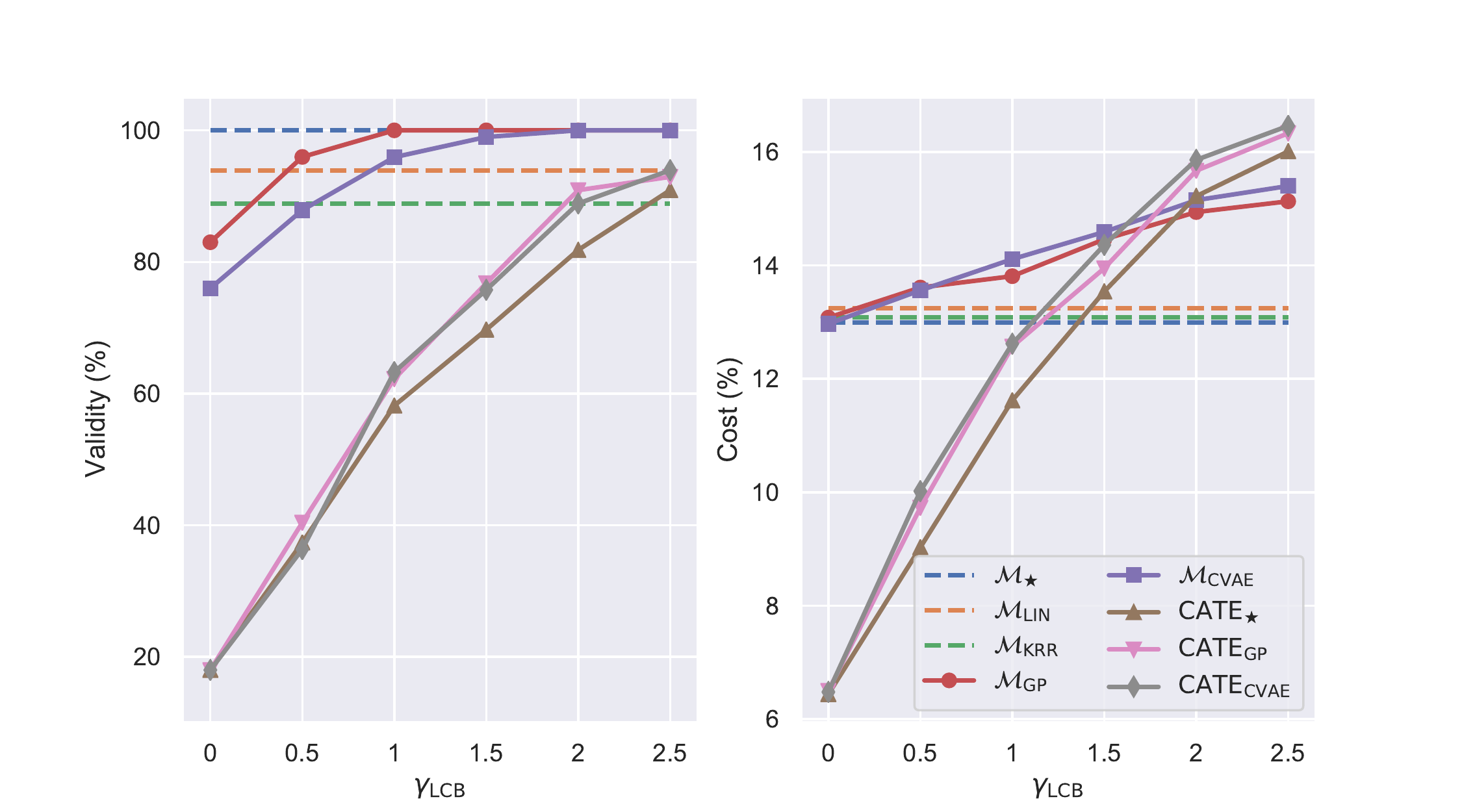}
        \subcaption{random forest}
        \label{fig:gamma_lcb_forest}
    \end{minipage}
    \caption{Trade-off between validity and cost which can be controlled via $\gamma_\LCB$ for the probabilistic recourse methods. Shown is the same setting as in \fig\ref{fig:gamma_LCB} using instead a non-linear logistic regression in the form of a multilayer perceptron (MLP; left), and a random forest (right) as classifiers $h$.}
    \vspace{-1em}
    \label{fig:app_lcb_plots}
\end{figure}

\begin{table*}[h]
  \scriptsize
  \setlength{\tabcolsep}{1pt}
  \caption{Experimental results for the gradient-descent approach on different 3-variable \scm s (top to bottom: linear \scm, non-linear \anm, non-additive \scm). We show average performance for $N_\text{runs} = 100$, $N_\text{MC-samples}=100$, and $\gamma_\LCB=2$, and display the number (out of $N_\text{runs}$) of performed interventions on all subsets of variables by each recourse type. The two right-most columns display how many of the intervention sets for each recourse type agreed with the suggestions made by the oracle methods, $\Mstar$ and $\CATEstar$, respectively. 
  We observe that interventions proposed by the subpopulation-based oracle often differ from the ones proposed at the individual level, which can be visually explained by \fig\ref{fig:illustration}.
  Importantly, we observe general agreement among all \cate\ approaches in their selection of intervened-upon variables. In contrast, 
  we observe that individual-based methods deviate away from their oracle (i.e., $\Mstar$) in their selection of variables to intervene upon for recourse.  This result further suggest that the \cate\ approaches presented in this work exhibit more predictable behaviour, as they are less sensitive to model assumptions, and are thus more preferable for the individual seeking recourse under imperfect causal knowledge.
    }
  \label{table:3-vars-app-grad-descent-frequency}
  \begin{tabular}{l c c c c c c c c c c c c}
    \toprule
    \multirow{2}{*}{Method} & \multicolumn{3}{c}{\scm}                 & \multicolumn{7}{c}{\textsc{Intervention Set}} & \multicolumn{2}{c}{\textsc{Identical Int. Set}}           \\ \cmidrule(r){2-4} \cmidrule(r){5-11} \cmidrule(r) {12-13}
                           & $\Validstar$ (\%)  &  $\LCB$  & Cost (\%) & $\{X_1\}$ & $\{X_2\}$ & $\{X_3\}$ & $\{X_1,X_2\}$ & $\{X_1,X_3\}$ & $\{X_2,X_3\}$ & $\{X_1,X_2,X_3\}$ & $\Mstar$ & $\CATEstar$  \\
    \midrule
  
    $\Mstar$               & 100  & -           & 10.9$\pm$7.9         &  0        & 25        & 0         & 56            & 0             &  0            & 19                & 100      &  23          \\
    $\MH_{\lin}$           & 100  & -           & 11.0$\pm$7.0         &  0        & 26        & 0         & 50            & 0             &  1            & 23                &  52      &  23          \\
    $\MH_\KR$              &  90  & -           & 10.7$\pm$6.5         &  0        & 22        & 0         & 44            & 0             &  0            & 34                &  54      &  27          \\
    $\MH_{\gp}$            & 100  & .55$\pm$.04 & 12.2$\pm$8.3         &  0        &  6        & 0         & 13            & 0             &  7            & 74                &  25      &  61          \\
    $\MH_{\cvae}$          & 100  & .55$\pm$.07 & 11.8$\pm$7.7         &  0        & 12        & 0         & 25            & 0             &  5            & 58                &  31      &  57          \\
    $\CATEstar$            &  90  & .56$\pm$.07 & 11.9$\pm$9.2         &  0        &  6        & 0         & 11            & 0             & 13            & 70                &  23      & 100          \\
    $\cate_\gp$            &  93  & .56$\pm$.05 & 12.2$\pm$8.4         &  0        &  3        & 0         &  9            & 1             & 15            & 72                &  18      &  76          \\
    $\cate_\cvae$          &  89  & .56$\pm$.08 & 12.1$\pm$8.9         &  0        &  6        & 1         & 11            & 0             & 16            & 66                &  18      &  78          \\

    \midrule
  
    $\Mstar$               & 100  & -           & 20.1$\pm$12.3        & 70        &  0        & 0         &  2            &16             &  0            & 11                &  99      &  17          \\
    $\MH_{\lin}$           &  54  & -           & 20.6$\pm$11.0        & 13        &  0        & 0         &  0            &81             &  0            &  5                &  20      &  41          \\
    $\MH_\KR$              &  91  & -           & 20.6$\pm$12.5        & 65        &  0        & 0         &  1            &23             &  0            & 10                &  76      &  22          \\
    $\MH_{\gp}$            & 100  & .54$\pm$.03 & 21.9$\pm$12.9        & 39        &  0        & 0         &  0            &38             &  0            & 22                &  54      &  38          \\
    $\MH_{\cvae}$          &  97  & .54$\pm$.05 & 22.6$\pm$12.3        & 33        &  0        & 0         &  0            &51             &  0            & 15                &  45      &  42          \\
    $\CATEstar$            &  97  & .55$\pm$.05 & 26.3$\pm$21.4        &  4        &  0        & 0         &  0            &44             &  2            & 49                &  17      &  99          \\
    $\cate_\gp$            &  94  & .55$\pm$.06 & 25.0$\pm$14.8        &  4        &  1        & 0         &  0            &37             &  4            & 53                &  11      &  69          \\
    $\cate_\cvae$          &  98  & .54$\pm$.05 & 26.0$\pm$14.3        &  3        &  0        & 0         &  1            &32             &  1            & 62                &  12      &  70          \\

    \midrule
  
    $\Mstar$               & 100  & -           & 13.2$\pm$11.0        &  0        &  0        & 1         &  0            &11             & 78            &  7                &  97      &  78          \\
    $\MH_{\lin}$           &  98  & -           & 14.0$\pm$13.5        &  0        &  0        & 0         &  1            & 0             & 85            & 11                &  81      &  77          \\
    $\MH_\KR$              &  70  & -           & 13.2$\pm$11.6        &  0        & 17        & 0         &  4            &10             & 59            &  7                &  55      &  53          \\
    $\MH_{\gp}$            &  95  & .52$\pm$.04 & 13.4$\pm$12.8        &  3        &  1        & 2         &  0            & 0             & 82            &  9                &  73      &  78          \\
    $\MH_{\cvae}$          &  95  & .51$\pm$.01 & 13.4$\pm$12.2        &  0        &  3        & 1         &  5            & 2             & 71            & 15                &  72      &  76          \\
    $\CATEstar$            & 100  & .52$\pm$.02 & 13.5$\pm$13.0        &  0        &  0        & 2         &  0            & 9             & 77            &  9                &  78      &  97          \\
    $\cate_\gp$            &  94  & .52$\pm$.03 & 13.2$\pm$13.1        &  3        &  1        & 5         &  0            & 3             & 73            & 12                &  70      &  76          \\
    $\cate_\cvae$          & 100  & .52$\pm$.05 & 13.6$\pm$12.9        &  0        &  1        & 2         &  0            & 1             & 82            & 11                &  78      &  78          \\
    
    \bottomrule
  \end{tabular}
  \vspace{-1em}
\end{table*}

\clearpage
\section{(Non-)identifability of \scm s under different assumptions}
\label{app:nonidentifiability}
In general form, i.e., without any further assumption on the structural equations $\mathbf{S}$ or noise distribution $P_\U$, \scm s are not identifiable from data alone, meaning that there are multiple different \scm s (possibly with different underlying causal graphs) which imply the same observational distribution \cite{peters2017elements}.
One possible construction relies on the use of the inverse cumulative distribution function (cdf) in combination with uniformly-distributed random variables \cite{darmois1951analyse} and is also used in non-identifiability proofs for non-linear independent component analysis (ICA)~\citep{hyvarinen1999nonlinear}.
Even knowing the causal graph is generally not enough as summarised in the following proposition.

\begin{proposition}
Even when the causal graph is known, the conditionals $P(X_r|\X_{\pa(r)})$ alone are insufficient to uniquely determine the structural equations $X_r:=f_r(\X_{\pa(r)}, U_r)$ without further assumptions.
\begin{proof}
This can be shown by using the following argument from~\cite[Footnote 1]{janzing2010causal} (adapted to our notation):
\begin{center}
 ``\textit{let $U_r$ consist of (possibly uncountably many) real-valued random variables $U_r[\x_{\pa(r)}]$, one for each value $\x_{\pa(r)}$ of the parents $\X_{\pa(r)}$.
Let $U_r[\x_{\pa(r)}]$ be distributed according to $P_{X_r\given \x_{\pa(r)}}$ and define $f_r(\x_{\pa(r)}, U_r):=U_r[\x_{\pa(r)}]$. Then $X_r\given \X_{\pa(r)}$ has distribution $P_{X_r\given \X_{pa(r)}}$}''.   
\end{center}
We can now build on this formulation to construct a second \scm\ with the same observational distribution and causal graph, e.g., by shifting the noise variables and structural equations by some fixed constant $C$ as follows.

For $r\in [d]$, define $Y_r:=X_r-C$.
Let $\Tilde{U}_r$ consist of (possibly uncountably many) real-valued random variables $\Tilde{U}_r[\x_{\pa(r)}]$, one for each value $\x_{\pa(r)}$ of the parents $\X_{\pa(r)}$.
Let $\Tilde{U}_r[\x_{\pa(r)}]$ be distributed according to $P_{Y_r\given \x_{\pa(r)}}$ and define $f_r(\x_{\pa(r)}, \Tilde{U}_r):=\Tilde{U}_r[\x_{\pa(r)}]+C$.
Then $X_r\given \X_{\pa(r)}$ also  has distribution $P_{X_r\given \X_{pa(r)}}$, but for $C\neq 0$ the structural equations and noise distributions are different from the previous construction.
\end{proof}
\end{proposition}

In the case of the \cvae-\scm\ model from~\eqref{eq:CVAE-SCM} the setting is slightly less general than the above, since we additionally assume that: (i) the noise distributions are isotropic multivariate Gaussian distributions of fixed dimension, $\z_r\sim\N_{d_{\z_r}}(\0,\Id)$; and (ii) the structural equations $D_r$ are from the class of functions that can be expressed as feedforward neural networks if fixed width and depth with learnable parameters $\psi_r$.

Unfortunately, we are not aware of any identifiability results for this particular setting, and further investigation into this matter is beyond the scope of the current work.
It is interesting to note, however, that the \cvae-\scm\ from~\eqref{eq:CVAE-SCM} can be understood as a non-linear extension of the linear Gaussian model with equal error variances considered by~\cite{peters2014identifiability}, for which identifiability has been shown.

In general, there seem to be very few works addressing identifiability of \scm s in the non-linear case; we refer to~\cite[][\S 7.1]{peters2017elements} for an overview of existing results.
Of particular interest for our setting is the  post-nonlinear model of~\cite{zhang2009identifiability}, which refers to the setting in which a non-linearity $g$ is applied on top of an \anm, i.e., $X_r:=g_r(f_r(\X_{\pa(r)}) + U_r)$, and for which complete conditions on $\{f_r, g_r\}$ have been provided that lead to identifiability.
Given the form of the decoders $D_r$---feedforward neural networks with stacked layers of simple non-linearities applied to linear transformations of the previous layers' output---it may be possible that the \cvae-\scm\ from~\eqref{eq:CVAE-SCM} can be interpreted as a nested post-nonlinear model.
We consider this an interesting direction, but leave further investigations into this matter for future work.

\clearpage
\section{Further details on \cvae\ training}
\label{app:cvae_training_details}
To learn the \cvae\ latent variable models, we perform amortised variational inference with approximate posteriors $q$ parameterised by encoders $\enc_r$ in the form of neural nets with parameters $\phi_r$,
\begin{equation}
    \label{eq:approx_posterior}
    p_{\psi_r}(\z_r\given x_r, \x_{\pa(r)})\approx q_{\phi_r}(\z_r\given x_r, \x_{\pa(r)}):=\N(\hat{\mu}_r, \hat{\sigma}_r^2), \quad \quad (\hat{\mu}_r, \hat{\sigma}_r^2) := \enc_r(x_r, \x_{\pa(r)}; \phi_r).
\end{equation}
The training objective in form of the evidence lower bound (ELBO) given data $\{\x^i\}_{i=1}^n$ is given by
\begin{equation}
    \label{eq:CVAE_objective}
    \mathcal{L}_r(\psi_r, \phi_r) 
    = \medop\sum_{i=1}^n \E_{q_{\phi_r}(\z|x_r^i, \x_{\pa(r)}^i)}\Big[\norm{x_r^i-\dec_r(\x_{\pa(r)}^i, \z; \psi_r)}^2\Big]+\beta_r \KLD{q_{\phi_r}(\z| x_r^i, \x_{\pa(r)}^i)}{p(z)}
\end{equation}
We learn both $\psi_r$ and $\phi_r$ simultaneously via stochastic gradient descend on $\mathcal{L}_r$, with gradients computed by Monte Carlo sampling from $q_{\phi_r}$ with reparametrisation.
Since the pairs of encoder and decoder parameters $(\psi_r, \phi_r)$ are independent for different $r$, this can be done in parallel.

\subsection{Hyperparameter selection for $\cvae$ training}

A $\cvae$ model was trained for every $\X_r|\X_{\pa(r)}$ relation. Generally, hyperparameters were selected by comparing the distribution of real samples from the dataset against reconstructed samples from the trained $\cvae$ obtained by sampling noise from the prior. The selection of hyperparameters was done either manually, or by performing a grid search over various encoder and decoder architectures, latent-space dimensions, and values of the hyperparameters $\beta_r$ that trade off the MSE and KL terms in the \cvae\ objective~\eqref{eq:CVAE_objective}.
For the case of automatic selection, the setup resulting in the smallest maximum mean discrepancy (MMD) statistic~\cite{gretton2012kernel} between real and reconstructed samples was chosen as  hyperparameter configuration.
Further details on the search space considered and the selected values are provided in Table~\ref{table:cvae_hyperparams}.

\newcommand{\tx}{$\times$}
\begin{table*}[h]
  \small
  \setlength{\tabcolsep}{4.5pt}
  \caption{Selection of hyperparameters for $\cvae$ training was either performed manually (for Linear \scm, Non-linear \anm, Non-additve \scm) or automatically (for 7-variable semi-synthetic loan approval) by selecting the setting that resulted in the minimum MMD statistic between real and reconstructed samples.}
  \label{table:cvae_hyperparams}
  \begin{tabular}{l l r r c c}
    \toprule
    \scm                                                                     & Conditional          & Encoder Arch.    &  Decoder Arch.   & Latent Dim.            & $\lambda_\text{KLD}$  \\
    \midrule
    
    \multirow{2}{*}{Linear \scm}                                             & $X_2|X_1,$           & 1\tx32\tx32\tx32 & 5\tx5\tx1        & 1                      & 0.01                  \\
                                                                             & $X_3|X_1,X_2$        & 1\tx32\tx32\tx32 & 32\tx32\tx32\tx1 & 1                      & 0.01                  \\
    \multirow{2}{*}{Non-linear \anm}                                         & $X_2|X_1,$           & 1\tx32\tx32      & 32\tx32\tx1      & 5                      & 0.01                  \\
                                                                             & $X_3|X_1,X_2$        & 1\tx32\tx32\tx32 & 32\tx32\tx1      & 1                      & 0.01                  \\
    \multirow{2}{*}{Non-additve \scm}                                        & $X_2|X_1,$           & 1\tx32\tx32\tx32 & 32\tx32\tx1      & 3                      & 0.5                   \\
                                                                             & $X_3|X_1,X_2$        & 1\tx32\tx32\tx32 & 5\tx5\tx1        & 3                      & 0.1                   \\

    \midrule
    
    \multirow{5}{*}{\parbox{3.5cm}{7-variable semi-synthetic loan approval}} & \multirow{5}{*}{any} &                  & 2\tx1            & \multirow{5}{*}{1,2} &                       \\
                                                                             &                      & 1\tx3\tx3        & 2\tx2\tx1        &                        & 5, 1, 0.5, 0.1,   \\
                                                                             &                      & 1\tx5\tx5        & 3\tx3\tx1        &                        &  0.05, 0.01,  \\
                                                                             &                      & 1\tx3\tx3\tx3    & 5\tx5\tx1        &                        &  0.005  \\
                                                                             &                      &                  & 3\tx3\tx3\tx1    &                        &                       \\
    \bottomrule
  \end{tabular}
  \vspace{-1em}
\end{table*}

\clearpage
\section{Experimental details, hyperparameter choices, and specification of \scm s}
\label{app:experimental_details}
\subsection{Specification of \scm s used in our experiments}
The following is a specification of all \scm s used in our experiments on synthetic and semi-synthetic data, both for data generation and to evaluate the validity of recourse actions proposed by the different approaches by computing the corresponding counterfactual in the ground-truth \scm s.

In addition, we also specify the model used to generate training labels. 
Note, however, that these labels are only used to train a new classifier (e.g., a logistic regression, multi-layer perceptron, or random forest) from scratch: this is the $h(\x)$ referred to in the main paper.
The label generating process is thus only used for obtaining labels to train a classifier on and is subsequently disregarded in favour of $h$.
 
In selecting the structural equations and label generating process, we tried to pick combinations that resulted in roughly centred features, as well as roughly balanced datasets (i.e., with a similar proportion of positive and negative training examples) that are not perfectly linearly-separable (i.e., with some class overlap).
Moreover, we tried to select settings that result in a diverse set of intervention targets selected by the oracle for different factual instances, i.e., we try to avoid situations in which the optimal action is to always intervene on the same (set of) variable(s).
To induce more interesting behaviour, we sample root nodes from mixtures of Gaussians.

\subsubsection{3-variable synthetic \scm s used for Table~\ref{table:3-vars-main}}

A visual summary of the 3-variable synthetic \scm s used for Table~\ref{table:3-vars-main} is provided in \fig\ref{fig:3-var-data}.

\begin{figure}[h]
    \begin{subfigure}[b]{0.33\textwidth}
        \centering
        \includegraphics[width=\textwidth]{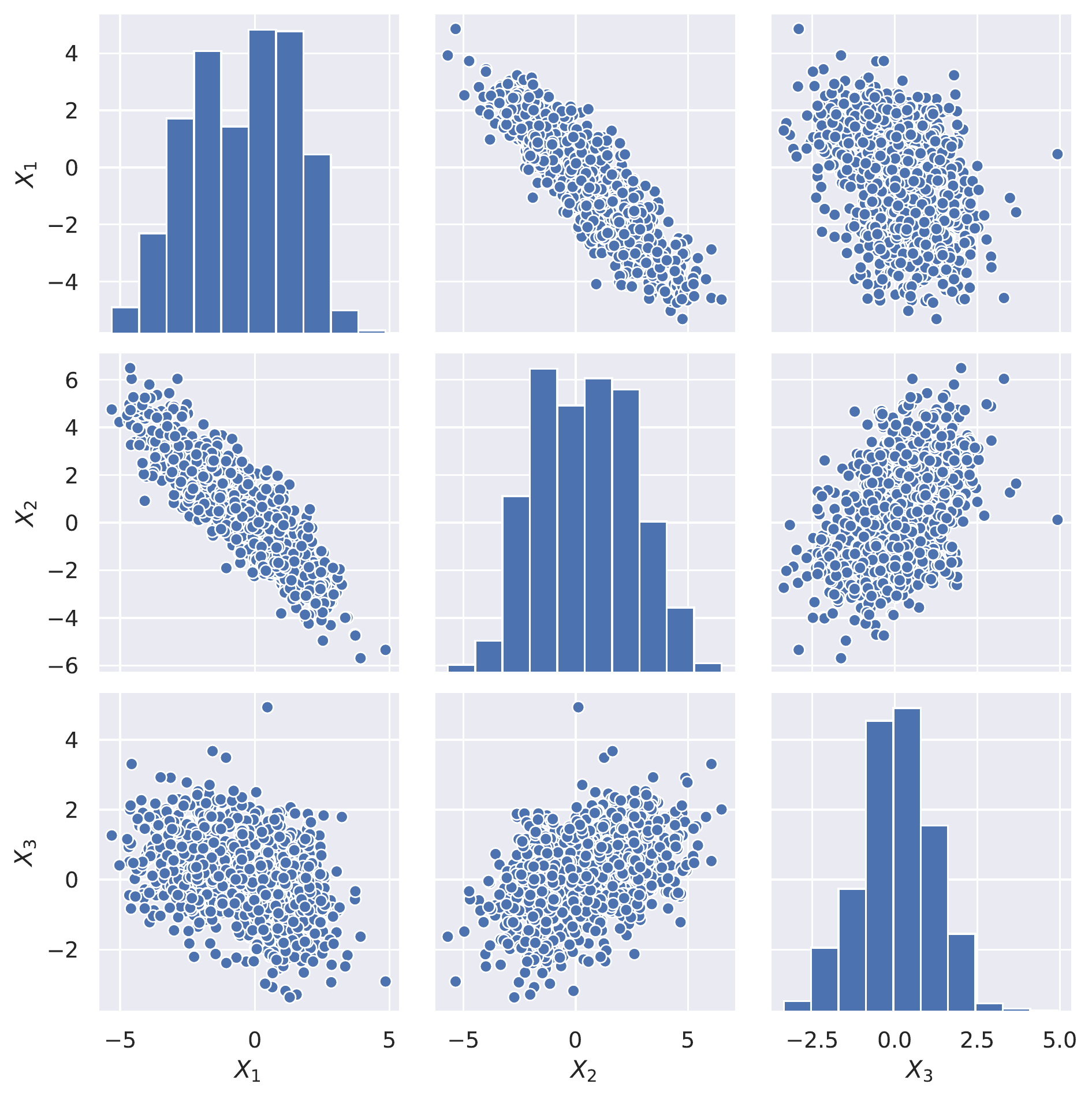}
        \caption{Linear \scm}
        \label{fig:sanity-3-lin}
    \end{subfigure}
    \begin{subfigure}[b]{0.33\textwidth}
        \centering
        \includegraphics[width=\textwidth]{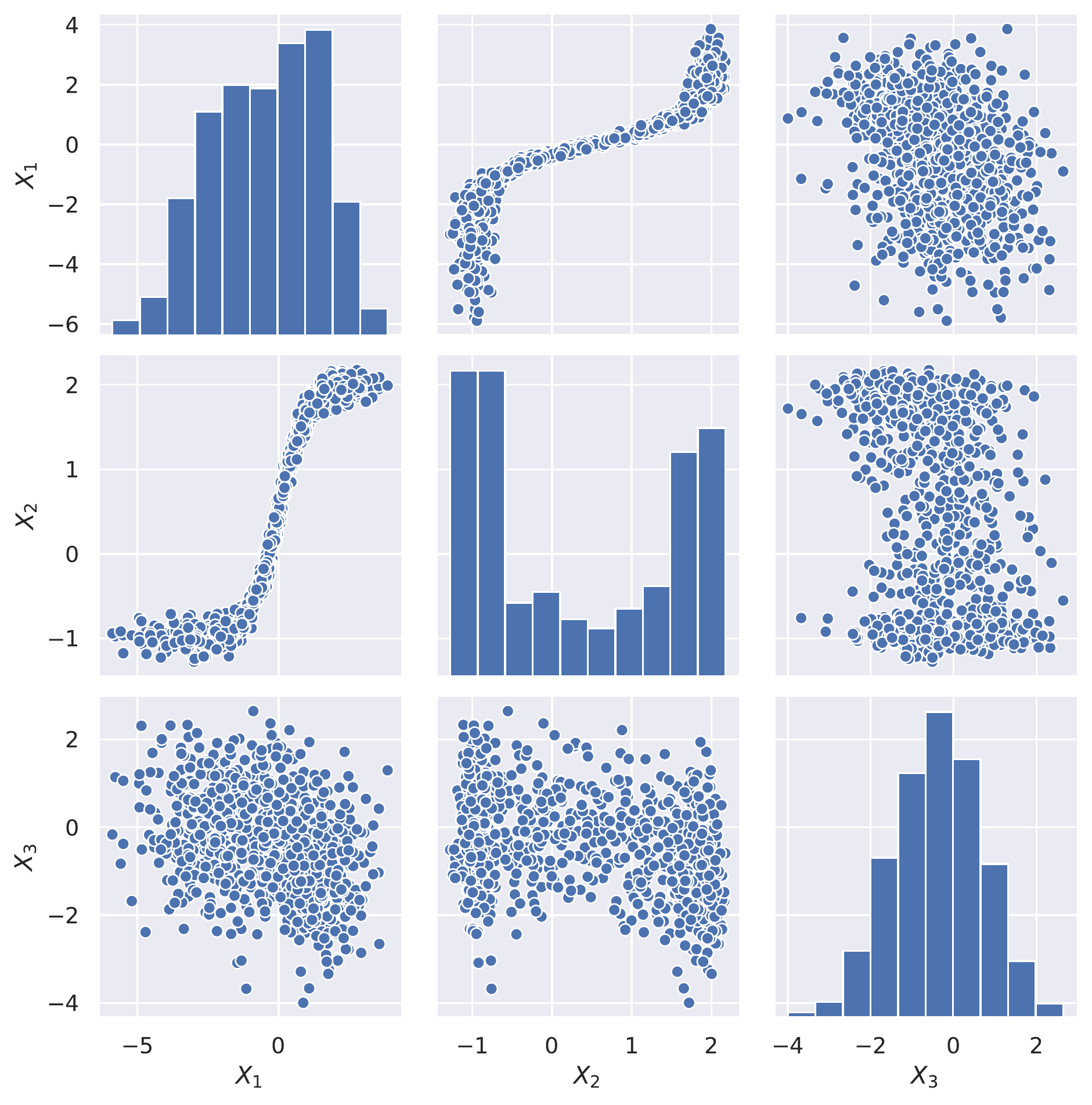}
        \caption{Non-linear \anm}
        \label{fig:sanity-3-anm}
    \end{subfigure}%
    \begin{subfigure}[b]{0.33\textwidth}
        \centering
        \includegraphics[width=\textwidth]{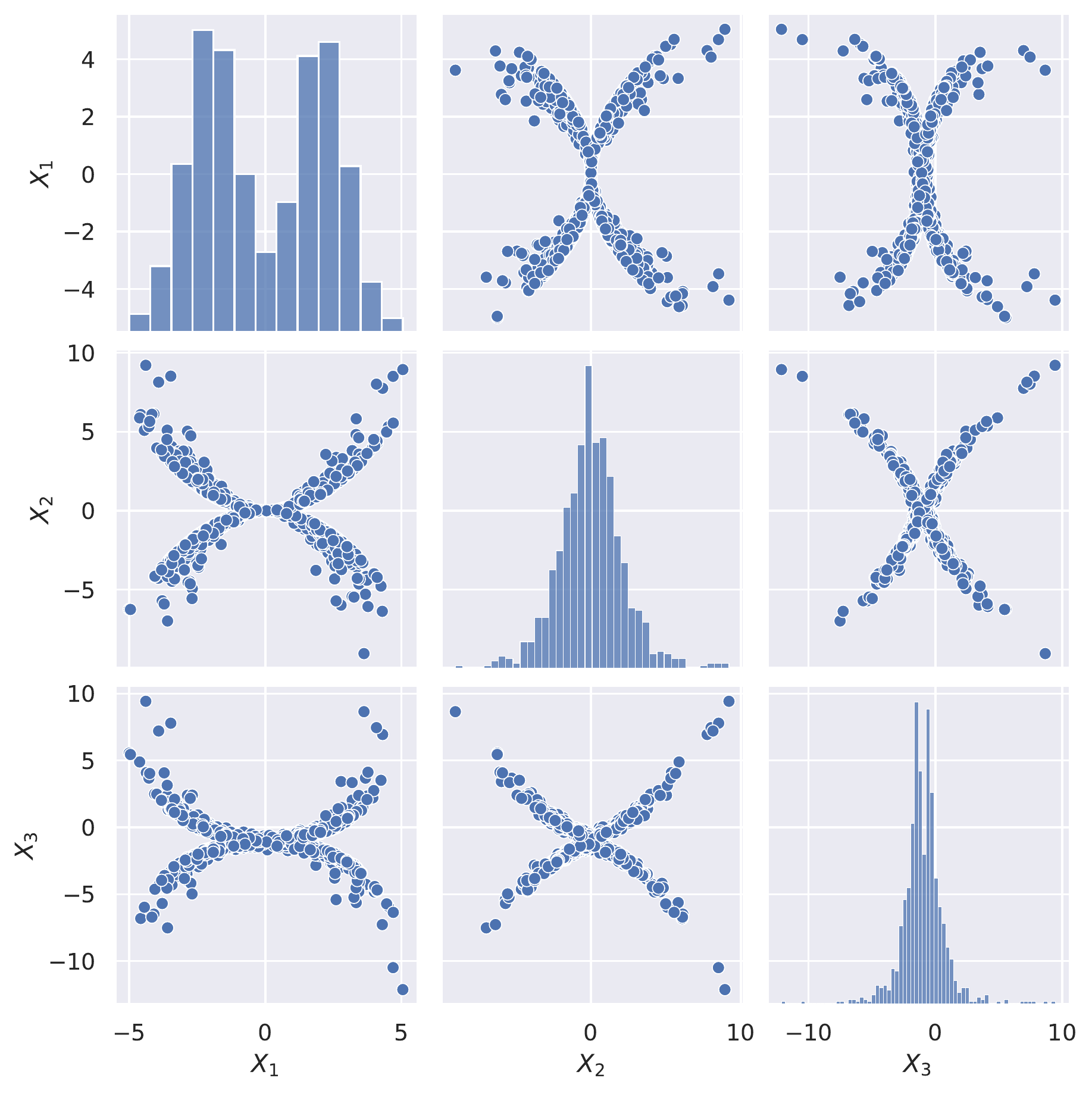}
        \caption{Non-additive \scm}
        \label{fig:sanity-3-gen}
    \end{subfigure}
    \caption{Histograms and scatter plots of pairwise feature relations for the synthetic 3-variable \scm s.}
    \label{fig:3-var-data}
\end{figure}

\paragraph{Linear \scm:}
The linear 3-variable \scm\ consists of the following structural equations and noise distributions:
\begin{align}
    X_1 &:= U_1,  &U_1&\sim \text{MoG}\Big(0.5\N(-2, 1.5)+0.5\N(1, 1)\Big)\\
    X_2 &:= -X_1+U_2,  &U_2&\sim \N(0, 1)\\
    X_3 &:= 0.05X_1 +0.25X_2+U_3, &U_3&\sim\N(0, 1) 
\end{align}

\paragraph{Non-linear \anm:}
The non-linear 3-variable \anm\ consists of the following structural equations and noise distributions:
\begin{align}
    X_1 &:= U_1,  &U_1&\sim \text{MoG}\Big(0.5\N(-2, 1.5)+0.5\N(1, 1)\Big)\\
    X_2 &:= -1+\frac{3}{1+e^{-2X_1}}+U_2,  &U_2&\sim \N(0, 0.1)\\
    X_3 &:= -0.05X_1 +0.25X_2^2+U_3, &U_3&\sim\N(0, 1) 
\end{align}

\paragraph{Non-additve \scm:}
The non-additive 3-variable \scm\ consists of the following structural equations and noise distributions:
\begin{align}
    X_1 &:= U_1,  &U_1&\sim \text{MoG}\Big(0.5\N(-2.5, 1)+0.5\N(2.5, 1)\Big)\\
    X_2 &:= 0.25\, \text{sgn}(U_2)X_1^2(1+U_2^2),  &U_2&\sim \N(0, 0.25)\\
    X_3 &:= -1 + 0.1\, \text{sgn}(U_3)(X_1^2+X_2^2)+U_3, &U_3&\sim\N(0, 0.25^2) 
\end{align}

\paragraph{Label generation:}
For all 3-variable \scm s, labels $Y$ were sampled according to
\begin{equation}
    Y\sim \text{Bernoulli}\left(\left(1+e^{-2.5\rho^{-1}(X_1+X_2+X_3)}\right)^{-1}\right)
\end{equation}
where $\rho$ is the average of $(X_1+X_2+X_3)$ across all training samples.

\subsubsection{7-variable semi-synthetic loan approval \scm\ used for Table~\ref{table:german-credit}}
For the semi-synthetic dataset, we wanted to capture some relations between the involved variables that seemed somewhat intuitive to us and to some limited extent reflect a loan approval setting in the real-world:
\begin{itemize}
\item loan amount and duration being largest for mid-aged people who may want to build a house and start a family, and smaller for younger and older people;
\item loan duration increasing with loan amount due to the an upper limit on monthly payments that can be afforded
\item savings increasing once income passes a certain (minimal-sustenance)  threshold;
\item income increasing with age;
\item education increasing with age initially before eventually saturating;
\item gender differences in income and (access to) education due to existing gender-discrimination and inequality of opportunities in the population; 
\end{itemize}

A visual summary of the 7-variable semi-synthetic loan \scm is shown in \fig\ref{fig:7-var-data}.

\begin{figure}
    \centering
    \includegraphics[width=\textwidth]{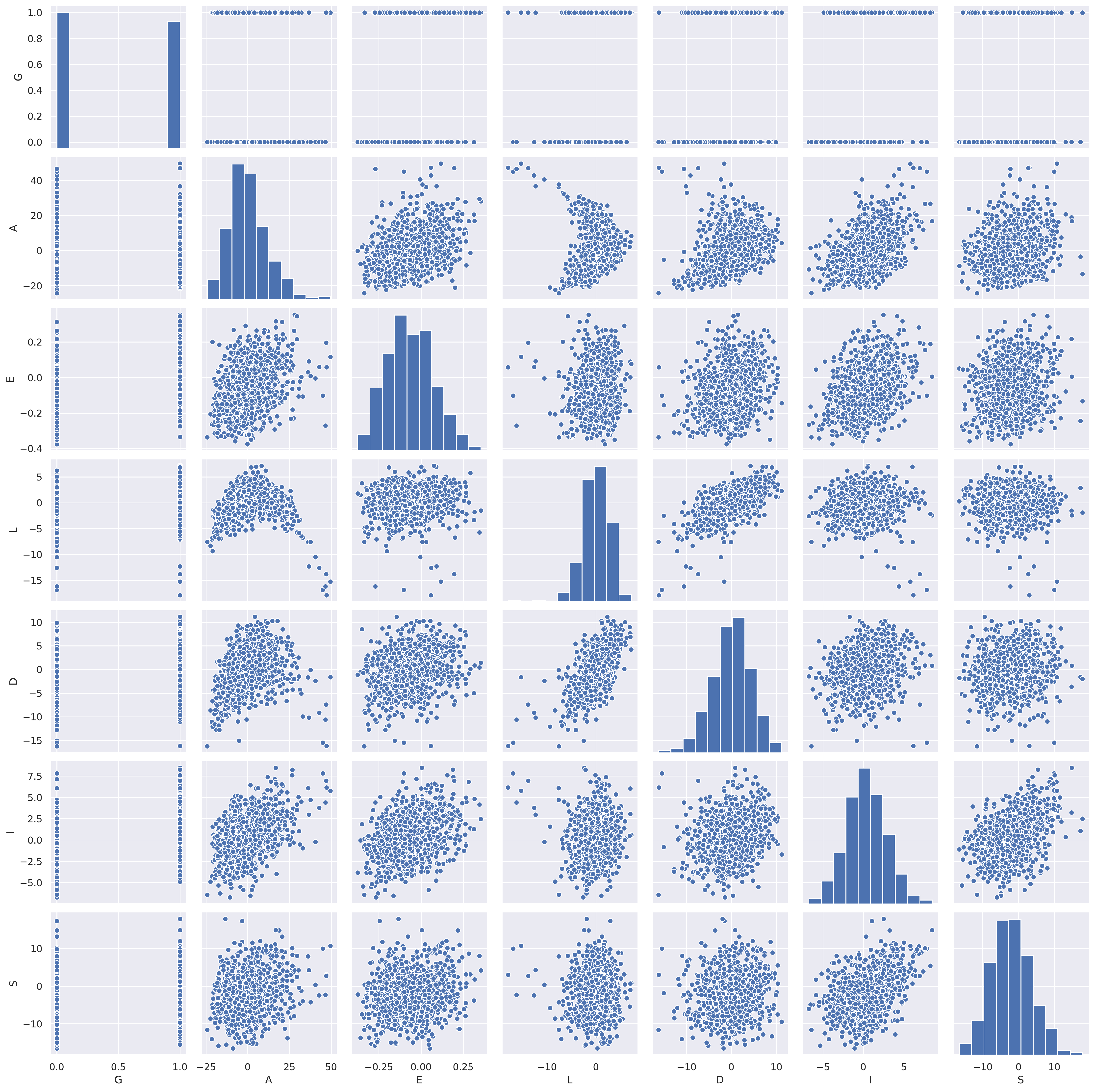}
    \caption{Histograms and scatter plots of pairwise feature relations for the semi-synthetic loan \scm.}
    \label{fig:7-var-data}
\end{figure}

\paragraph{Semi-synthetic \scm:}
The loan approval \scm\ consists of the following structural equations and noise distributions:
\begin{align}
    G &:= U_G,  &U_G&\sim \text{Bernoulli}(0.5)\\
    A &:= -35+U_A,  &U_A&\sim \text{Gamma}(10, 3.5)\\
    E &:= -0.5 + \left(1 + e^{-\left(-1 + 0.5 G + \left(1 + e^{- 0.1 A}\right)^{-1} + U_E\right)}\right)^{-1},  &U_E&\sim \N(0,0.25)\\
    L &:= 1 + 0.01 (A - 5) (5 - A) + G + U_L, &U_L&\sim \N(0,4)\\
    D &:= -1 + 0.1A + 2G + L + U_D, &U_D&\sim \N(0, 9)\\
    I &:= -4 + 0.1(A + 35) + 2G + G E + U_I, &U_I&\sim \N(0, 4)\\
    S &:=  -4 + 1.5 \mathbb{I}_{\{I > 0\}} I + U_S, &U_S&\sim\N(0, 25)
\end{align}
Note that variables in the above \scm\ often have a relative meaning in terms of deviation from the mean, e.g., we centre the Gamma-distributed age around its mean of 35, so that $A$ has the meaning of ``age-difference from the mean of 35'' (and similarly for other variables).

\paragraph{Label generation:}
Labels $Y$ were sampled according to
\begin{equation}
    Y\sim \text{Bernoulli}\left(\left(1+e^{-0.3(-L-D+I+S+IS)}\right)^{-1}\right).
\end{equation}
Note that this label generation process only depends on loan duration and amount, income and savings, but not on gender, age or education level.

\clearpage
\section{Derivation of a Monte-Carlo estimator for the gradient of the variance}
\label{app:derivation_of_variance_grad}
We now derive an estimator for the gradient of the square-root of the variance (i.e., standard deviation) of $h$ over the interventional or counterfactual distribution of $\X_{\d(\I)}$ w.r.t.\ $\th$, which appears (multiplied by $\lambda_\LCB$) in the threshold $\texttt{tresh}(a)$ of the optimisation constraint/regulariser.

First, we use the chain rule of differentiation to write
\begin{align}
\label{eq:app_gradient_standard_deviation}
\nabla_\th \sqrt{\mathbb{V}_{\X_{\d(\I)}}\left[h\left(\X_{\d(\I)}, \th, \xF_{\nd(\I)}\right)\right]}
&= \frac
{\nabla_\th \mathbb{V}_{\X_{\d(\I)}}\left[h\left(\X_{\d(\I)}, \th, \xF_{\nd(\I)}\right)\right]}
{2\sqrt{\mathbb{V}_{\X_{\d(\I)}}\left[h\left(\X_{\d(\I)}, \th, \xF_{\nd(\I)}\right)\right]}}
\end{align}

Next, we write the variance as expectation and---assuming the interventional or counterfactual distribution of $\X_{\d(\I)}$ admits reparametrisation as is the case for the \gpscm\ and \cvae\ models used in this paper---use the reparametrisation trick to differentiate through the expectation operator as in \eqref{eq:reparametrisation_trick}.
\begin{align}
\nabla_\th
&\mathbb{V}_{\X_{\d(\I)}}
\Big[
h\big(\X_{\d(\I)}, \th, \xF_{\nd(\I)}\big)
\Big]\\
&= \nabla_\th \E_{\X_{\d(\I)}}\left[
\left(
h\left(\X_{\d(\I)}, \th, \xF_{\nd(\I)}\right)
-\E_{\X_{\d(\I)}'}
\Big[
h\left(\X_{\d(\I)}', \th, \xF_{\nd(\I)}\right)
\Big]
\right)^2
\right]\\
&=
\nabla_\th \E_{\z\sim\N(\0,\Id)}
\left[
\Big(
h\left(\X_{\d(\I)}(\z;\th), \th, \xF_{\nd(\I)}\right)
-\E_{\z'\sim\N(\0,\Id)}
\Big[
h\left(\x_{\d(\I)}(\z';\th), \th, \xF_{\nd(\I)}\right)
\Big]
\Big)^2
\right]\\
&=
\E_{\z\sim\N(\0,\Id)}\left[
\nabla_\th
\Big(
h\left(\X_{\d(\I)}(\z;\th), \th, \xF_{\nd(\I)}\right)
-\E_{\z'\sim\N(\0,\Id)}
\Big[
h\left(\x_{\d(\I)}(\z';\th), \th, \xF_{\nd(\I)}\right)
\Big]
\Big)^2
\right]\\
&=
\E_{\z\sim\N(\0,\Id)}
\Bigg[
2
\bigg(
h\left(\X_{\d(\I)}(\z;\th), \th, \xF_{\nd(\I)}\right)
-\E_{\z'\sim\N(\0,\Id)}
\Big[
h\left(\x_{\d(\I)}(\z';\th), \th, \xF_{\nd(\I)}\right)
\Big]
\bigg)\Bigg.
\\
& \Bigg.\times
\bigg(
\nabla_{\th}  h\left(\X_{\d(\I)}(\z;\th), \th, \xF_{\nd(\I)}\right)
-\E_{\z'\sim\N(\0,\Id)}
\Big[
\nabla_{\th}
h\left(\x_{\d(\I)}(\z';\th), \th, \xF_{\nd(\I)}\right)
\Big]
\bigg)
\Bigg]
\end{align}

We can now obtain an estimate of the gradient with two independent sets of Monte Carlo samples of $\X_{\d(\I)}$, drawn via reparametrisation from the interventional or counterfactual distribution,
\begin{equation}
\{\x_{\d(\I)}^{(m)}:= \x_{\d(\I)}(\z^{(m)};\th)\}_{m=1}^M, \quad \{\x_{\d(\I)}^{(m')}:= \x_{\d(\I)}(\z^{(m')};\th)\}_{m'=1}^{M'} \quad \text{where} \quad \z^{(m)},\z^{(m')}\overset{\text{i.i.d.}}{\sim} \N(\0,\Id).
\end{equation}

This yields the following Monte Carlo gradient estimator of the variance:
\begin{align}
\nabla_\th
\mathbb{V}_{\X_{\d(\I)}}
&
\Big[
h\big(\X_{\d(\I)}, \th, \xF_{\nd(\I)}\big)
\Big]
\approx \frac{1}{M}\sum_{m=1}^M
\Bigg[
2
\bigg(
h\left(\x_{\d(\I)}^{(m)}, \th, \xF_{\nd(\I)}\right)
-\frac{1}{M'}\sum_{m'=1}^{M}
h\left(\x_{\d(\I)}^{(m')}, \th, \xF_{\nd(\I)}\right)
\bigg)\Bigg.
\\
& \Bigg.\times
\bigg(
\nabla_{\th}  h\left(\x_{\d(\I)}^{(m)}, \th, \xF_{\nd(\I)}\right)
-\frac{1}{M'}\sum_{m'=1}^{M'}
\nabla_{\th}
h\left(\x_{\d(\I)}^{(m')}, \th, \xF_{\nd(\I)}\right)
\bigg)
\Bigg]
\end{align}

Substituting the above expression, together with the following Monte Carlo estimate of the (undifferentiated) variance
 \begin{equation}
    \mathbb{V}_{\X_{\d(\I)}}\left[h\left(\X_{\d(\I)}, \th, \xF_{\nd(\I)}\right)\right]
    \approx
    \frac{1}{M-1}\sum_{m=1}^M 
    \bigg(
    h\left(\x_{\d(\I)}^{(m)}, \th, \xF_{\nd(\I)}\right)
    - 
    \frac{1}{M}\sum_{m'=1}^{M'}
    h\left(\x_{\d(\I)}^{(m')}, \th, \xF_{\nd(\I)}\right)
    \bigg)^2,
 \end{equation}
into \eqref{eq:app_gradient_standard_deviation} gives the desired estimate for the gradient of the standard deviation of $h$.

\end{document}